\documentclass[11pt]{article}
\usepackage[margin=1.25in]{geometry}

\usepackage{shortcuts}
\usepackage{amsthm,amsmath,amssymb}
\usepackage{hyperref}
\usepackage{graphicx}
\usepackage[capitalise]{cleveref}
\Crefname{assumption}{Assumption}{Assumptions}
\Crefname{condition}{Condition}{Conditions}
\usepackage[authoryear,sort,round]{natbib}
\usepackage{color}
\newcommand{\norm}[1]{\left\lVert#1\right\rVert}
\newcommand{\pr}{\mathbb{P}}
\newcommand{\expect}{\mathbb{E}}
\newcommand{\ind}{\mathbb{I}}
\newcommand{\blockedit}{\color{magenta}}
\renewcommand\S{\mathcal S}
\newcommand\A{\mathcal A}
\newcommand\D{\mathcal D}

\newcommand\F{\mathcal F}
\newcommand{\RNum}[1]{\uppercase\expandafter{\romannumeral #1\relax}}

\newcommand\tr{^\intercal}
\theoremstyle{plain}
\newtheorem{theorem}{Theorem}
\newtheorem{lemma}[theorem]{Lemma}
\newtheorem{corollary}[theorem]{Corollary}
\theoremstyle{definition}
\newtheorem{assumption}{Assumption}
\newtheorem{condition}{Condition}

\newenvironment{myproof}[1][Proof]{\begin{proof}[#1]}{\end{proof}}

\usepackage{xcolor}

\newcommand*\samethanks[1][\value{footnote}]{\footnotemark[#1]}
\title{Fast Rates for the Regret of Offline Reinforcement Learning}
\author{Yichun Hu\thanks{Alphabetical order.},~~~~Nathan Kallus\samethanks,~~~~Masatoshi Uehara\samethanks\\Cornell University}
\date{}

\def\blockedit{}
\def\edit{}

\begin{document}

\maketitle

\begin{abstract}%
We study the regret of reinforcement learning from offline data generated by a fixed behavior policy in an infinite-horizon discounted Markov decision process (MDP). While existing analyses of common approaches, such as fitted $Q$-iteration (FQI), suggest a $O(1/\sqrt{n})$ convergence for regret, empirical behavior exhibits \emph{much} faster convergence. In this paper, we present a finer regret analysis that exactly characterizes this phenomenon by providing fast rates for the regret convergence. First, we show that given any estimate for the optimal quality function $Q^*$, the regret of the policy it defines converges at a rate given by the exponentiation of the $Q^*$-estimate's pointwise convergence rate, thus speeding it up. The level of exponentiation depends on the level of noise in the \emph{decision-making} problem, rather than the estimation problem. We establish such noise levels for linear and tabular MDPs as examples. Second, we provide new analyses of FQI and Bellman residual minimization to establish the correct pointwise convergence guarantees. As specific cases, our results imply $O(1/n)$ regret rates in linear cases and $\exp(-\Omega(n))$ regret rates in tabular cases. We extend our findings to general function approximation by extending our results to regret guarantees based on $L_p$-convergence rates for estimating $Q^*$ rather than pointwise rates, where $L_2$ guarantees for nonparametric $Q^*$-estimation can be ensured under mild conditions.
\end{abstract}

\section{Introduction}

Offline reinforcement learning (RL) is the problem of learning a reward-maximizing policy in an unknown Markov decision process (MDP) from data generated by running a fixed policy in the same MDP.
The problem is particularly relevant in applications where exploration is limited but observational data plentiful. Medicine is one such example: ethical, safety, and operational considerations limit both the application of unproven or random policies and the running of online-updating algorithms, while at the same time rich electronic health records are collected en-masse.

A variety of methods have been proposed for offline RL including fitted $Q$-iteration (FQI) \citep{ernst2005tree,munos2008finite}, fitted policy iteration \citep{LagoudakisMichail2004LPI,antos2008learning,liu2020provably}, modified Bellman Residual Minimization \citep{antos2008learning,ChenJinglin2019ICiB}, SBEED \citep{pmlr-v80-dai18c}, and MABO \citep{XieTengyang2020QASf}.
For all of these, the regret (value suboptimality) bounds obtained are $O(1/\sqrt{n})$, where $n$ is the number of observed transition data (see for example Chapter 15 of \citep{agarwal2019reinforcement} for a concise presentation of the analysis of FQI).
However, in practice, the regret convergence can actually be \emph{much} faster. For example, we provide a linear-MDP simulation experiment where FQI empirically exhibits an apparent regret convergence rate of $O(1/n)$.

In this paper, we tightly characterize this phenomenon by theoretically establishing fast rates for the regret convergence of value-based offline RL methods, which directly estimate the optimal quality function, $Q^{*}$. These rates leverage the specific noise level of a given problem instance, expressed as the density near zero of the suboptimality of the second-best action (if any), also known as a \emph{margin condition}.
RL instances generally satisfy \emph{some} instance-specific nontrivial margin condition. We moreover show that in the linear and tabular cases, we generally have quite strong margin conditions. We show that policies that are greedy with respect to good estimates of $Q^*$ enjoy a regret bounded by the pointwise estimation error raised to a power \emph{larger} than one, thus speeding up convergence for the downstream decision-making task. This analysis can be applied to any value-based offline RL method that has pointwise convergence guarantees for estimating $Q^{*}$. As specific examples, we establish that we can achieve such pointwise error bounds for the linear case using FQI and modified BRM (differently from existing analyses of their average error). Together, this means that, under the standard assumptions needed for FQI and modified BRM, \ie, closedneess under Bellman operators (completeness) and sufficient feature coverage, linear FQI and modified BRM generally achieve regret of order $O(1/n)$ in linear MDPs. Technically, our analysis melds techniques from fast-rate analysis of classification \citep{audibert2007fast} with the theoretical analyses of RL \citep{agarwal2019reinforcement} and of empirical risk minimization \citep{wainwright2019high}. {\blockedit Finally, we extend our results to accommodate any value-based offline RL method that provides $L_p$-guarantees over the offline data for estimating $Q^*$ in place of pointwise error guarantees. This is particularly relevant for $p=2$ (MSE), for which FQI and other methods can ensure $L_2$-guarantees (instead of pointwise) for general function approximation rather than just linear models.}
\subsection{Set Up}

We consider a time-homogeneous, finite-action, infinite-horizon, discounted MDP.
Namely, we have an arbitrary measurable state space $\S$ (\eg, continuous, discrete, or other), a finite actions space $\A$ (\ie, $\abs{\A}<\infty$), a reward distribution $P_r(\cdot\mid s,a)$ that maps to a probability measure on $\Rl$, a transition kernel $P_s(\cdot\mid s,a)$ that maps to a probability measure on $\S$,
an initial state distribution $\mu$ on $\S$, and a discount factor $0<\gamma<1$.
We let $r(s,a)$ denote the mean of $P_r(\cdot\mid s,a)$.

When we play a policy $\pi(a\mid s)$ in this MDP, the trajectory $s_0,a_0,r_0,s_1,a_1,r_1,\dots$ is given the distribution $s_0\sim \mu$, $a_0\sim\pi(\cdot\mid s_0)$, $r_0\sim P_r(\cdot\mid s_0,a_0)$, $s_1\sim P_s(\cdot\mid s_0,a_0)$, $a_1\sim\pi(\cdot\mid s_1)$, \ldots.
Since we consider different policies in the same MDP,
we refer to this distribution as $\pr^\pi$ and expectations over it as $\E^\pi$.
With some abuse of notation, we also identify maps $\pi:\S\to\A$ with the deterministic policy given by Dirac at $\pi(s)$ (\ie, $a_t=\pi(s_t)$).
For each policy $\pi$, we define the  $Q$-function, $V$-function, and average state occupancy distribution, respectively, as
\begin{align*}
    Q^{\pi}(s,a) & = \expect^\pi\bracks{\sum_{t=0}^{\infty} \gamma^t r(s_t, a_t)\mid s_0 = s,a_0=a},\\
    V^{\pi} (s) & = \expect^\pi\bracks{\sum_{t=0}^{\infty} \gamma^t r(s_t, a_t)\mid s_0 = s},\\
    d^{\pi}(S) & = (1-\gamma) \sum_{t=0}^{\infty} \gamma^t \pr^{\pi}(s_t \in S)\quad\text{for measurable $S$}.
\end{align*}
The reward of a policy $\pi$ is
$$
V^\pi=\E^\pi\bracks{\sum_{t=0}^\infty \gamma^t r_t}=\E_{s\sim d^\pi,a\sim\pi(\cdot\mid s)}[r(s,a)].
$$
We also define the optimal value, optimal $V$-function, and optimal $Q$-function, respectively, as
$$
V^*=\max_{\pi}V^\pi,~V^*(s)=\max_{\pi}V^\pi(s),~
Q^*(s,a)=\max_{\pi}Q(s,a).
$$
We always let $\pi^*$ be any deterministic policy with $\pi^*(s)\in\arg\max_{a\in\A}Q^*(s,a)$. Notice that $V^*=V^{\pi^*}$, $V^*(s)=V^{\pi^*}(s)$, $Q^*(s,a)=Q^{\pi^*}(s,a)$.

We also define the Bellman optimality operator: for $f:\mathcal S\times\mathcal A\to\Rl$,
$$
\mathcal{T}f(s,a) = r(s,a) + \gamma \expect_{s'\sim P_s(\cdot\mid s,a)} \max_{a'\in\mathcal{A}} f(s',a').
$$
Notice that $Q^*$ is the unique fixed point of $\mathcal T$ (up to measure zero states).

Two types of MDPs we will sometimes use just as specific examples are tabular and linear MDPs. A \emph{tabular MDP} is one with finite state space (we already assume the action space is finite).
A \emph{linear MDP} \citep{jin2020provably} is one where for some known $\phi:\mathcal S\times\mathcal A\to\R d$ 
with $\magd{\phi(s,a)}\leq1$
and unknown vector $\theta\in\R d$ and measures $\nu=(\nu_1,\dots,\nu_d)$, we have
$$
r(s,a)=\theta\tr\phi(s,a),\quad P_s(S\mid s,a)=\nu(S)\tr\phi(s,a)~~\text{for measurable $S$}.
$$

\paragraph{Notation}
All unsubscripted norms, $\magd\cdot$, are Euclidean norms.
For a function $f(s,a)$ we define $\|f\|_\infty=\sup_{s\in\S,a\in\A}\abs{f(s,a)}$.
Additional norms will be defined as needed. For a square symmetric matrix $A$ we let $\lambda_{\min}(A)$ be its smallest eigenvalue. 

\subsection{The Offline Reinforcement Learning Problem}
The learning problem is as follows.
The MDP is unknown and we only observe transitions data from some (possibly unknown) stochastic policy, known as the behavior policy. Namely, for some (possibly unknown) $\mu_b$, we observe $n$ independent and identically distributed (iid) draws $\mathcal D=\{(s_i,a_i,r_i,s_i'):i=1,\dots,n\}$ where each follows
$$
(s_i,a_i)\sim\mu_b,~r_i\sim P_r(\cdot\mid s_i,a_i),~s'_i\sim P_s(\cdot\mid s_i,a_i).
$$
We let $\pr_\D$ and $\E_\D$ denote the probability and expectation with respect to the random sampling of the data $\mathcal D$.
Based on this data, we choose a data-driven policy $\hat\pi$. The target is to find one with small average regret,
$$
\E_\D\bracks{V^*-V^{\hat\pi}}.
$$
In particular, we will focus on $Q$-greedy policies that, given some $f(s,a)$, are given by any $\pi_f(s)\in\arg\max_{a\in\A}f(s,a)$.
In particular, results will hold for any choice of tie breaking.
Note, if $f=Q^*$ then $\pi_f=\pi^*$. Given some hypothesis class $\F$, we also define $\Pi_\F=\{\pi_f:f\in\F\}\subseteq[\S\to\A]$.

\subsection{Fitted $Q$-Iteration} \label{section: fqi}

We will use FQI as one example of the regret behavior of offline RL. We present modified BRM as an additional example in \cref{sec:sbeed}.

The FQI algorithm is as follows \citep{ernst2005tree}:
\begin{enumerate}
\item Start at any $\hat f_0:\mathcal S\times\mathcal A\to\Rl$ (\eg, the zero function).
\item For $k=1,\dots,K$:
\begin{enumerate}
    \item Set $y_i=r_i+\gamma \max_{a'\in\A}\hat f_{k-1}(s'_i,a')$.
    \item Use any supervised learning algorithm to regress $y_i$ on $(s_i,a_i)$ to obtain $\hat f_k$. \label{alg: step 2b}
\end{enumerate}
\item Return $\hat f_K$ and $\hat\pi=\pi_{\hat f_K}$.
\end{enumerate}

When the supervised learning algorithm is given by empirical risk minimization of squared loss over a hypothesis class $\F$ (\ie, least squares), that step can be written as
\begin{equation}\label{eq:fqils}
\hat f_k\in\arg\min_{f\in\F}\sum_{i=1}^n\prns{f(s_i,a_i)-r_i-\gamma \max_{a'\in\A}\hat f_{k-1}(s'_i,a')}^2.
\end{equation}

\begin{figure}[t!]
\begin{center}
  \includegraphics[width=0.7\textwidth]{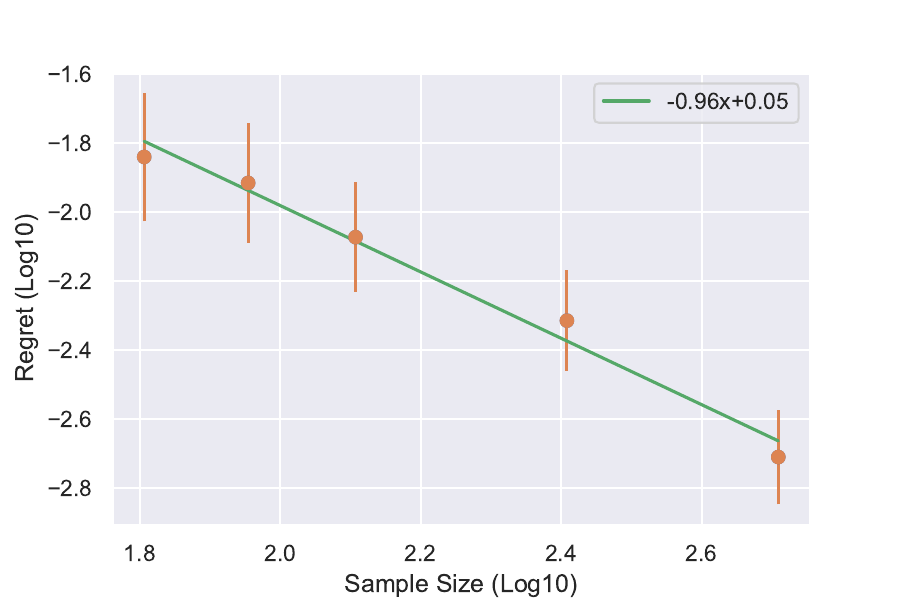}
\caption{The regret of FQI in a simple example on a log-log scale, along with a linear trend fit}
\label{fig:regret}
\end{center}
\end{figure}

\subsection{The Empirical Performance of FQI Belies Current Analysis}\label{sec:simpleexample}

We next study the performance of FQI in a simple example of a linear MDP. We construct the underlying MDP as follows. We set $\mathcal{S}=[0,1]^2,\,\mathcal{A}=\{0,1\}$. We set the initial distribution $\mu(\cdot)=\mathrm{Unif}_{[0,1]^2}(\cdot)$, where $\mathrm{Unif}_{[0,1]^2}(\cdot)$ is the uniform distribution on $[0,1]^2$.
We let $\phi(s,a)=(s_1(1-a),s_1a,1-s_1,s_2(1-a),s_2a,1-s_2)/2$, which belongs to the simplex in $\R 6$. We then set $r(s,a)=\theta\tr\phi(s,a)$ and
$$P_{s}(s'\mid s,a)=\sum_{k=1}^{6}\phi_k(s,a)\op{Beta}_{10\alpha_{k1},\,10\beta_{k1}}(s_1')\times \op{Beta}_{10\alpha_{k2},\,10\beta_{k2}}(s_2'),$$
where $\op{Beta}_{\alpha,\,\beta}(\cdot)$ is a Beta distribution.
We fix some $\theta,\alpha_{1,1},\beta_{1,1},\dots,\alpha_{6,2},\beta_{6,2}$ by drawing from $\op{Unif}_{[0,1]^{30}}$ (once, after setting the random seed to zero).
We let the behavior policy be uniform: $\mu_b(s,a)=\mu(s)\times\op{Ber}_{{0.5}}(a)$, where $\op{Ber}_{0.5}(\cdot)$ is a Bernoulli distribution with parameter $0.5$. We set the discount as $\gamma=0.9$. 

We then apply FQI to data generated by the above MDP and behavior policy using $K=50$ iterations and a linear function class, $\F=\{\beta\tr \phi(s,a):\beta\in\R 6\}$ (\ie, \cref{eq:fqils} becomes ordinary least squares). We vary the sample size $n\in[64,90,128,256,512]$ and run $70$ replications of the experiment for each sample size, in each replication recording the resulting regret $V^{*}-V^{\bar \pi }$.
Namely, we calculate the value of a given policy by running the policy on an independent sample of $40000$ initial states and truncating at step $50$, and we compute $\pi^*$ by running FQI with $K=100$ iterations on another independent dataset of size $n=40000$.
\edit{Since, FQI with $K=\Omega(\log(n))$ converges at a rate of $O(1/n)$ as we later show, running FQI with $n=40000,\,K=100$ provides a very good approximation for $\pi^*$ for benchmarking.}
We report the results in \cref{fig:regret} on a logarithmic scale along with $75\%$-confidence intervals for each $n$ and a linear trend fit to the log-log-transformed data.
The empirically observed slope with $75$\%-confidence interval is $-0.96\pm 0.08$, which is somewhat suggestive of a regret rate of roughly $O(1/n)$. This provides concrete empirical evidence that for some instances, we may be able to get regret convergence that is much faster than the $O(1/\sqrt n)$ appearing in the existing analyses of FQI and other offline RL algorithms. 

\section{Fast Rates for $Q$-Greedy Policies} \label{sec:q greedy}

In this section, we show that any estimate $\hat f$ of $Q^*$
with some rate of convergence leads to a $Q$-greedy policy 
with regret rate that is the \emph{exponentiation} of this estimation rate, and sometimes even an exponential to this rate.
This can possibly speed up the rate considerably.
The level of exponentiation depends on the level of noise in the 
downstream \emph{decision} problem 
(rather than in the $Q$-estimation problem), 
that is, how hard is it to distinguish
optimal actions from near-optimal actions (rather than how hard it is to estimate $Q^*$), also known as a margin in classification and bandit problems.

We define the margin at $s$ as
$$
\Delta(s) = \pw{
\max_{a\in \mathcal{A}} Q^*(s,a) - \max_{a\notin\argmax_{a\in\A}Q^*(s,a)} Q^*(s,a)\quad&\argmax_{a\in\A}Q^*(s,a)\neq \mathcal{A}\\0&\argmax_{a\in\A}Q^*(s,a)= \mathcal{A}}.$$
The margin can be smaller at some $s$ and larger at other $s$.
The larger the margin, the clearer is the choice of the optimal action, the easier it is to learn to make this optimal choice.
However, the margin may well be positive and arbitrarily close to $0$ in many continuous settings (while a $0$ margin leads to trivial decision making).
So, motivated by related conditions in classification \citep{mammen1999smooth,tsybakov2004optimal,audibert2007fast} and multi-arm contextual bandits \citep{perchet2013multi,hu2022smooth,bastani2020online},
we use the following condition to describe the \emph{density} of $\Delta(s)$ near (but not at) zero.
\begin{condition}[Margin] \label{ass: margin}
Fix some class of deterministic policies $\Pi\subseteq[\S\to\A]$.
There exist constants $\delta_0 >0, \alpha\in[0,\infty]$ such that for all $\delta>0$,
\begin{align*}
    \sup_{\pi\in\Pi}\pr_{s\sim d^{\pi}}(0<\Delta(s) \le \delta) \le (\delta/\delta_0)^{\alpha},
\end{align*}
where $x^\infty$ is understood as $0$ for $x\in[0,1)$, $1$ for $x=1$, and $\infty$ for $x>1$.
\end{condition}

We can often just take $\Pi=[\S\to\A]$ to be all deterministic policies for simplicity, but it will be sufficient to take only $\Pi=\Pi_\F$ when using a hypothesis class $\F$ for learning $Q^*$.
All instances satisfy \cref{ass: margin} with $\alpha=0$.
But, generally, a given instance would satisfy \cref{ass: margin} with some $\alpha>0$.
At one extreme, if $\Delta(s)$ is uniformly bounded away from $0$ over $s$ then \cref{ass: margin} holds with $\alpha=\infty$. 
We give examples where we can establish a margin below in \cref{sec:marginexamples}

Our result applies to $Q$-greedy policies given a good estimate $\hat f$ of $Q^*$. Our next condition quantifies the quality of the estimate.

\begin{condition}[Pointwise error bound] \label{cond:ptwse}
  A data-driven $\hat{f}(s,a)$ is given such that for some $C>0$ and $a_n>0$ it satisfies that, for any $(s,a) \in \{(s,a) \in \mathcal{S}\times \mathcal{A}:\exists \pi \in \Pi, d_{\pi}(s,a)>0  \}$ and $\delta\ge a_n$, we have
        \begin{align} \label{eq: infty convergence}
          \pr_{\mathcal{D}}(|\hat{f}(s,a)- Q^*(s,a)| \ge \delta) \le C \exp(-\delta^2/a_n^2).
    \end{align}
\end{condition}
\Cref{eq: infty convergence} is a \emph{pointwise} convergence bound for $\hat f$ with rate $a_n$. In our second main result, in \cref{sec:fqi}, we will show that linear FQI satisfies an even stronger \emph{uniform} convergence bound (with the supremum over $s,a$ inside the probability) with $a_n=1/\sqrt{n}$. 
In \cref{sec:sbeed}, we show similar results for Bellman residual minimization.  

In \cref{sec:sbeed}, we show similar results for Bellman residual minimization.

Using \cref{ass: margin,cond:ptwse}, we can now establish our key rate-speed-up result.
\begin{theorem} \label{thm: fast rate}
Let a data-driven $\hat f$ be given and let $\Pi$ be given such that $\pi_{\hat f}\in\Pi$ almost surely.
        Suppose \cref{ass: margin,cond:ptwse} hold and $\norm{Q^*}_{\infty} \le Q_{\max}$.
        Fix any $\delta_1 \ge \delta_0$ with $\delta_1>2a_n$.
        Let $i_{\max} $ be the largest integer such that $2^{i_{\max}+1}a_n < \delta_1$. 
        Then, for 
        $\hat\pi=\pi_{\hat f}$,
        we have
        \begin{align*}
            \expect_\D[V^* - V^{\hat{\pi}}] &\le  \frac{2^{\alpha+1} }{(1-\gamma) \delta_0^{\alpha}} \prns{1  + C \sum_{i=1}^{i_{\max}}  \exp\prns{- 2^{2i-2}}2^{(\alpha+1)i+1}} a_n^{\alpha+1} \\
            &\phantom{\le}+ \frac{2Q_{\max}C}{1-\gamma} \exp\prns{-  \delta_1^2/(4a_n)^2}.
        \end{align*}
\end{theorem}

If either $\Pi=[\S\to\A]$ or $\Pi=\Pi_\F$ and $\hat f\in\F$ then obviously $\pi_{\hat f}\in\Pi$ is satisfied. 
Note that we could also easily only require that $\pi_{\hat f}\in\Pi$ with high probability and the failure probability would simply propagate into the regret bound.
The key to the result is to leverage the performance difference lemma \citep[Lemma 1.16]{agarwal2019reinforcement} together with a peeling argument to study the behavior at different scales of margin.

We immediately have the following two corollaries to simplify the above expression, one for the case $\alpha<\infty$ and one for the case $\alpha=\infty$.
\begin{corollary}\label{cor:finitemargin}
    Suppose $\pi_{\hat f}\in\Pi$ almost surely, \cref{cond:ptwse} holds, and \cref{ass: margin} holds with $\alpha<\infty$. 
        Then, for 
        $\hat\pi=\pi_{\hat f}$,
        we have
        \begin{align*}
            \expect_\D[V^* - V^{\hat{\pi}}] \le \frac{2^{\alpha+1} }{(1-\gamma) \delta_0^{\alpha}} \prns{1  +   c(\alpha) C} a_n^{\alpha+1} ,
        \end{align*}
where $c(\alpha) = \sum_{i=1}^{\infty}  \exp\prns{- 2^{2i-2}}2^{(\alpha+1)i+1} \le \frac{2^{\alpha+1}\Gamma\prns{\frac{\alpha+1}{2},1}}{\log 2} + 2\prns{\frac{2(\alpha+1)}{e}}^{(\alpha+1)/2}$, with $\Gamma\prns{\frac{\alpha+1}{2},1} = \int_1^{\infty} x^{\frac{\alpha-1}{2}}e^{-x} dx<\infty$.
\end{corollary}%
\Cref{cor:finitemargin} shows that the \emph{estimation} rate in \cref{cond:ptwse} gets sped up by exponentiation by $1+\alpha$ when applied to the downstream \emph{decision making} problem. Thus, however fast we are able to estimate $Q^*$, our regret converges \emph{even faster}. 
Notice that in \cref{cor:finitemargin} we do not actually need to assume a bound on $\magd{Q^*}_\infty$.
\edit{Heuristically, \cref{cor:infinitemargin} is obtained by taking $\delta_1 = \infty$ in \cref{thm: fast rate} and noting that the term that involves $Q_{\max}$ in the upper bound becomes $0$. The actual proof simply involves redoing the analysis of \cref{thm: fast rate} with $\delta_1 = \infty$ so that $i_{\max}=\infty$ and never encountering the final term where we needed to leverage the $Q_{\max}$-bound.}

\begin{corollary}\label{cor:infinitemargin}
    Suppose $\pi_{\hat f}\in\Pi$ almost surely, \cref{cond:ptwse} holds, \cref{ass: margin} holds with $\alpha = \infty$, and $\norm{Q^*}_{\infty} \le Q_{\max}$. Let $n$ be such that $a_n<\delta_0/2$.
    Then, for 
    $\hat\pi=\pi_{\hat f}$,
    we have
    \begin{align*}
        \expect_\D[V^* - V^{\hat{\pi}}] \le \frac{2Q_{\max}C}{1-\gamma} \exp\prns{-   \delta_0^2/(4a_n)^2}.
    \end{align*}  
\end{corollary}
\Cref{cor:infinitemargin} shows that in the $\alpha=\infty$ case, our regret vanishes \emph{exponentially} fast.
While \cref{cor:finitemargin,cor:infinitemargin} provide a simple understanding of the behavior in $n$ at any fixed $\alpha$, if $\alpha$ is finite but very big (\eg, a regime where $\alpha=\omega(1)$ with respect to $n$) then \cref{thm: fast rate} with $\delta_1=\delta_0$ 
characterizes the correct trade-off between the polynomial and exponential terms.

\subsection{The Margin Condition in Some Examples}\label{sec:marginexamples}

We next discuss some cases where we can explicitly demonstrate a nontrivial margin condition.
The heuristic implication is that we should \emph{generically} expect $\alpha=1$ in continuous-state settings and $\alpha=\infty$ in tabular settings. 

The next lemma shows that if $Q^*$ is linear and under a kind of weak concentratability assumption, we have $\alpha=1$.
\begin{lemma}\label{lemma:linearmargin}
Suppose $Q^*(s,a)=\beta_a\tr\psi(s)$ for some $\psi:\S\to\R d$ with $\magd{\psi(s)}\leq 1$ and $\beta\in\R {\A\times d}$, and that $\psi(s)$ with $s\sim d^\pi$ has a density for each $\pi\in\Pi$ and this density is bounded by $\mu_{\max}$.
Then, \cref{ass: margin} holds with $\alpha=1$ and $\delta_0=\fprns{6\mu_{\max}\sum_{a\in\A}\max_{a'\in\A:\beta_a\neq\beta_{a'}}\magd{\beta_a-\beta_{a'}}^{-1}}^{-1}$.
\end{lemma}

A linear MDP is sufficient for the condition on $Q^*$ as we can take $\psi(s)=(\phi(s,a)/\sqrt{\abs{\A}})_{a\in\A}$.
The assumption of uniformly bounded density is not especially restrictive. In offline RL, we often assume some type of overlap condition; for example, that $d^{\pi}$ and $\mu_b$ have densities (let us overload notation and call these densities $d^{\pi},\,\mu_b$) and that
$\sup_{\pi\in\Pi_\F}\|d^{\pi}(s)\pi(a\mid s)/\mu_b(s,a)\|_\infty\leq C_1$ (\eg, \citep{XieTengyang2020QASf}).
(See also \citep{pmlr-v32-scherrer14} for a discussion of various overlap conditions.) This implies that the condition in \cref{lemma:linearmargin} is satisfied with $\mu_{\max}=C_1\|\mu_b\|_\infty$, if $\|\mu_b\|_\infty<\infty$.

{\blockedit

The result essentially continues to hold even if $s\mapsto (Q^*(s,a))_{a\in\A}$ is nonlinear 
as long as the $Q$-difference functions have lower-bounded derivatives along one direction,
as we show next.
This suggests we can very generally expect to have $\alpha=1$ in practice in continuous-state-space settings.

\begin{lemma}\label{lemma: margin for continuous functions}
Assume $\mathcal{S}$ is contained in the $d$-dimensional unit ball and that, 
for each $\pi\in\Pi$, $s\sim d^\pi$ has a density and it is bounded by $\mu_{\max}$.
Suppose there exists a constant $\tau_0>0$ such that 
for any $a, a'\in \mathcal A$, there is some vector $v$ with $\norm{v} = 1$ such that for any $s\in \mathcal{S}$, the $v$-directional derivative of $Q^*(s,a) - Q^*(s, a')$ exists and is bounded below by $\tau_0$, that is,
\begin{align*}
    \abs{\lim_{t\to0}t^{-1}(Q^*(s+tv,a) - Q^*(s+tv, a') - Q^*(s,a) + Q^*(s, a'))} \ge \tau_0.
 \end{align*}
Then, \cref{ass: margin} holds with $\alpha=1$ and $\delta_0= \tau_0/(6 |\mathcal A|^2 \mu_{\max})$.
\end{lemma}

}

In the tabular setting, we trivially have $\alpha=\infty$, albeit with a $\delta_0$ that might be small. The following follows trivially by enumeration. 
\begin{lemma}\label{lemma:tabularmargin}
Suppose $\abs{\S}<\infty$. Set $\Pi=[\S\to\A]$.
Then, \cref{ass: margin} holds with $\alpha=\infty$ and $\delta_0=\max_{s,a, a':Q^*(s,a)\neq Q^*(s,a')}\abs{Q^*(s,a)-Q^*(s,a')}^{-1}$.
\end{lemma}

\section{Fast Rates for Linear Fitted $Q$-Iteration}\label{sec:fqi}

In this section, we show that by applying FQI with a linear function class, we can obtain an estimator $\hat f_K$ that has a pointiwse guarantee as in \cref{eq: infty convergence} with $a_n=O(1/\sqrt{n})$ . Thus, \cref{thm: fast rate} ensures that we will obtain regret of order $O(n^{-(1+\alpha)/2})$ when we execute the greedy policy ${\pi}_{\hat f_K}$. When $\alpha=1$, as in \cref{lemma:linearmargin}, this means we have regret $O(1/n)$, just as was observed empirically in \cref{sec:simpleexample}.

\paragraph{Bounded Linear Class and Assumptions}

Assume we are given a feature map $\phi: \mathcal{S}\times \mathcal{A} \rightarrow \mathbb{R}^d$ with $\|\phi(s,a)\|\le 1$. Throughout \cref{sec:fqi}, we will consider the hypothesis class $\mathcal{F}$ that is linear in these features with bounded coefficients:
\begin{align}\label{eq:Flinear}
    \mathcal{F} = \{ w\tr \phi(s,a): w\in \mathbb{R}^d,\, \|w\|\le B\}.
\end{align}
Here we introduce two assumptions to ensure the convergence of FQI estimators, which are commonly seen in literature (see Chapter 15 of \citep{agarwal2019reinforcement} for a review).

Our first assumption is that our function class $\mathcal{F}$ is closed under the Bellman operator $\mathcal{T}$.
\begin{assumption}[Completeness] \label{ass: completeness}
For any $f\in\mathcal{F}$, $\mathcal{T}f\in \mathcal{F}$.
\end{assumption}
Since $Q^*$ is the fixed point of $\mathcal{T}$, \cref{ass: completeness} directly implies realizability, i.e., $Q^*\in \mathcal{F}$.

As a special example, in the case of linear MDPs, 
for any $f\in \mathcal{F}$, we have
\begin{align*}
    \mathcal{T}f(s,a) &=  r(s,a) + \gamma \int_{\mathcal{S}} \max_{a'\in\mathcal{A}} f(s',a') dP(s'\mid s,a)\\
    &=  \theta\tr\phi(s,a) + \gamma \int_{\mathcal{S}} \max_{a'\in\mathcal{A}} f(s',a')    d\nu(s') \tr \phi(s,a)\\
    &=  \prns{\theta + \gamma \int_{\mathcal{S}} \max_{a'\in\mathcal{A}} f(s',a')d\nu(s')} \tr \phi(s,a),
\end{align*}
so $\mathcal{T}f$ is a linear function in $\phi$ as well.
Moreover, if $\max\{\|\nu(\mathcal{S})\|, \|\theta\|\}\le C_{\max}$
and $\gamma C_{\max} <1$, it is easy to see that $\mathcal{F} = \{ w\tr \phi(s,a): w\in \mathbb{R}^d,\, \|w\|\le C_{\max}/(1-\gamma C_{\max})\}$ satisfies \cref{ass: completeness}.

Our second assumption is to ensure sufficient feature coverage in our data set.
This assumption is commonly required to guarantee the convergence of ordinary least squares estimators \citep{wang2021statistical}.
\begin{assumption}[Feature Coverage] \label{ass: min eigen}
There exists a constant $\lambda_0>0$ such that
\begin{align*}
    \lambda_{\min} (\expect_{s,a\sim \mu_b} [\phi(s,a)\phi(s,a)\tr]) \ge \lambda_0.
\end{align*}
\end{assumption}

\paragraph{Uniform Convergence of the Linear FQI Estimator}
We now apply the FQI algorithm as is stated in \cref{section: fqi} to get $\hat{f}_K$ with $\mathcal{F}$.
Here we elaborate on how we regress $y_i$ on $(s_i,a_i)$ to obtain $\hat f_k$ in step \ref{alg: step 2b}\edit{; the procedure is slightly different from \cref{eq:fqils} to simplify the analysis.} 

For any $f\in \mathcal{F}$, by \cref{ass: completeness} there exists $w_f$ with $\norm{w_f} \le B$ such that $\mathcal{T}f = w_f\tr\phi$.
Define the empirical design matrix
\begin{align*}
   \hat{\Sigma} = \sum_{i=1}^n \phi(s_i, a_i) \phi(s_i, a_i)\tr.
\end{align*}
If  $\lambda_{\min}(\hat{\Sigma}) >0$, let $\hat{w}_f'$ be the OLS regressor
\begin{align*}
 \hat{w}_f' &=  \arg\min_{w\in \mathbb{R}^d} \sum_{i = 1}^n \prns{w\tr\phi(s_i,a_i) - r_i - \gamma  \max_{a'\in\mathcal{A}} f\prns{s_i',a'}}^2\\
 &=  \hat{\Sigma}^{-1}\prns{\sum_{i=1}^n \phi(s_i, a_i) \prns{r_i + \gamma  \max_{a'\in\mathcal{A}} f(s_i',a')}} , 
\end{align*}
and otherwise $\hat{w}'_f = 0$. 
\edit{Note that which value we set to $\hat{w}'_f$ when $\lambda_{\min}(\hat{\Sigma}) =0$ does not really matter under \cref{ass: min eigen} since such event happens with very low probably (see "bounding (a)" in the proof of \cref{lemma: convergence of T operator}). We just arbitrarily set it to $0$.} 
Finally, let $\hat{w}_f$ be the projection of $\hat{w}_f'$ on to a Euclidean ball $\mathcal{B}(0,B)$. 
We then set 
$$\hat{f}_k = \hat{w}_{\hat{f}_{k-1}}\tr \phi.$$

The following Lemma shows that with high probability, our one-step estimator $\hat{w}_f\tr\phi$ converges quickly to $\mathcal{T}f$, uniformly over all $(s,a)\in \mathcal{S}\times \mathcal{A}$ and all functions $f\in \mathcal{F}$.
In proving this Lemma, we leverage theoretical tools from matrix concentration and empirical process.
\begin{lemma} \label{lemma: convergence of T operator}
    Assume \cref{ass: completeness,ass: min eigen} hold and $|r_t| \le M$ for $t = 1, 2, \dots$. 
For any $\delta>0$, $\hat{w}_f$ estimated from the above procedure satisfies
\begin{align*}
    \pr\prns{\sup_{f\in\mathcal{F}}\norm{\hat{w}_f\tr\phi - \mathcal{T}f}_{\infty} \ge \delta }\le  
    6d\exp\prns{- \frac{\lambda_0^2}{5184 d^2 (M+B)^2 } n\delta^2}.
\end{align*}
\end{lemma}

Using \cref{lemma: convergence of T operator}, we can then obtain the following convergence guarantee for FQI.
\begin{theorem} \label{lemma: l infty fqi}
Assume \cref{ass: completeness,ass: min eigen} hold and $|r_t| \le M$ for $t = 1, 2, \dots$.
Set $a_n = \frac{144 d (M+B)}{(1-\gamma)\lambda_0 \sqrt{ n}}$.
Then, for any $K\ge \frac{\log(\lambda_0^2 n/(72 d)^2)}{2\log(1/\gamma)}$ and
any $\delta \ge a_n$,  we have
\begin{align*}
    \pr(\|Q^*- \hat f_K\|_{\infty}\ge \delta) \le 6d\exp(-\delta^2 /a_n^2).
\end{align*}
\end{theorem}

The key to showing \cref{lemma: l infty fqi} is establishing that (surely)
\begin{align*}
    \|Q^*- \hat f_K\|_{\infty} \le \sum_{t=0}^{K-1} \gamma^t\|\hat f_{K-t} - \mathcal{T} \hat f_{K-t-1}\|_{\infty} +  \frac{\gamma^K M}{1-\gamma}.
\end{align*}
This means that the estimation error of $\hat f_K$ can be controlled by two terms: one is a weighted sum of one-step estimation errors, which can be bounded using \cref{lemma: convergence of T operator}, and the other is a diminishing term as we increase the number of iterations.
Therefore, as long as our total number of iterations $K$ is large enough, with high probability $\hat f_K$ converges to $Q^*$ uniformly.

\paragraph{Fast Rates for Linear Fitted $Q$-Iteration}\label{sec:linear_fast}
Since uniform convergence is a stronger condition than pointwise convergence, \cref{lemma: l infty fqi} shows that $\hat{f}_K$ satisfies \cref{cond:ptwse} with $C = 6d$ and $a_n = \frac{144 d (M+B)}{(1-\gamma)\lambda_0 \sqrt{ n}}$.
Then, combined with \cref{cor:finitemargin,cor:infinitemargin} it immediately implies fast rates for linear fitted $Q$-interation. We summarize below. (For variable/large $\alpha$, we should apply \cref{thm: fast rate} to get the right trade off between the terms.)
\begin{corollary}[Fast Rates for Linear Fitted $Q$-Iteration] \label{thm: fast rate fqi}
    Suppose \cref{ass: margin} holds with $\Pi=\Pi_\F$, \cref{ass: min eigen,ass: completeness} hold, and $|r_t| \le M$ for $t = 1, 2, \dots$. 
    When $K\ge \frac{\log(\lambda_0^2 n/(72 d)^2)}{2\log(1/\gamma)}$,
    for $\hat{\pi}=\pi_{\hat{f}_K}$, we have:
\begin{enumerate}
       \item if $\alpha<\infty$,
       \begin{align*}
        \expect_\D[V^* - V^{\hat{\pi}}] \le \frac{288^{\alpha+1}\prns{1  +   6d c(\alpha)} }{(1-\gamma) \delta_0^{\alpha}}  \prns{\frac{ d (M+B)}{(1-\gamma)\lambda_0 }}^{\alpha+1} n^{-\frac{\alpha+1}{2}};
    \end{align*} 
       \item \edit{if $\alpha=\infty$,} %
       \begin{align*}
        \expect_\D[V^* - V^{\hat{\pi}}] \le \frac{12Md}{(1-\gamma)^2} \exp\prns{-  \prns{\frac{(1-\gamma)\lambda_0 \delta_0}{576d(M+B)}}^2 n }.
    \end{align*} 
\end{enumerate} 
\end{corollary}

\subsection{Tabular MDP as a Special Case}\label{sec:tabuular }

The tabular setting ($\abs{\S}<\infty$) is a special case of the linear MDP
since we can take $$\phi(s,a) = (\ind\{(s,a)=(s', a')\})_{(s',a')\in \mathcal{S}\times \mathcal{A}},\quad d = \abs{\mathcal{S}}\abs{\mathcal{A}}.$$ 
Moreover, we can satisfy \cref{ass: min eigen} with $\lambda_0 = \min_{(s,a)\in \mathcal{S}\times \mathcal{A}} \mu_b(s,a)$, and given $M$ s.t. $|r_{t}|\leq M$, we can satisfy \cref{ass: completeness} with $B=\sqrt{|S||A|}(1-\gamma)^{-1}M$. Thus, \cref{lemma: l infty fqi} gives us the bound of the model based estimator in a tabular case. This shows that with probability $1-\delta$, $\|\hat f_K-Q^{*}\|_{\infty}=\frac{144 \abs{\mathcal{S}}\abs{\mathcal{A}} (M+B)}{(1-\gamma)\lambda_0 }\sqrt{\frac{\log(6\abs{\mathcal{S}}\abs{\mathcal{A}}/\delta)}{n}}$. A typical way (\eg, \citep{SinghSatinderP1994Aubo}) to translate this into the regret of $\hat\pi=\pi_{\hat f_K}$ is to note $V^{*}-V^{\bar \pi }\leq \frac{1}{1-\gamma}\|\hat f_K-Q^{*}\|_{\infty}$. Integrating the tail gives
\begin{align*}
\E_\D[V^{*}-V^{\bar \pi }]\leq\frac{432\sqrt{\pi} \abs{\mathcal{S}}^2\abs{\mathcal{A}}^2(M+B)}{(1-\gamma)^2\lambda_0 \sqrt{n}}.
\end{align*}

Compared to this, our \cref{lemma:tabularmargin,lemma: l infty fqi,thm: fast rate fqi} together give that
\begin{align*}
 \expect_\D[V^* - V^{\hat{\pi}}] \le \frac{12M\abs{\mathcal{S}}\abs{\mathcal{A}}}{(1-\gamma)^2} \exp\prns{-  \prns{\frac{(1-\gamma)\lambda_0 \delta_0}{576 \abs{\mathcal{S}}\abs{\mathcal{A}}(M+B)}}^2 n },
\end{align*} 
where $\delta_0 = \max_{s,a, a':Q^*(s,a)\neq Q^*(s,a')}\abs{Q^*(s,a)-Q^*(s,a')}^{-1}$.
The above regret bound vanishes \emph{exponentially}, much faster than the usual $O(1/\sqrt{n})$ result.

\section{Fast Rates for Modified Bellman Residual Minimization}\label{sec:sbeed}

Modified Bellman Residual Minimization (BRM) is a common offline $Q$-function estimation method \citep{antos2008learning}, which approximates the Bellman error by introducing another maximization problem, thus avoiding the need to iterate. The original BRM was for offline policy evaluation (estimate $Q^\pi$ for a given $\pi$); recently \citet{ChenJinglin2019ICiB} adapted it to offline policy learning (estimate $Q^*$) in a method called MSBO. They establish the convergence of the Bellman residual errors of MSBO when the hypothesis classes are finite ($\abs{\F}<\infty$). %
In this section, we show the convergence of MSBO in terms of uniform error to $Q^*$ for a linear function class. Using our results, we conclude that MSBO enjoys fast rates as well, similarly to FQI. 

Given, a class $\F_w\subseteq[\S\times\A\to\Rl]$ and $\zeta>0$, MSBO is defined as follows: 
\begin{align*}
    \hat f\in \argmin_{q \in \mathcal{F}}\max_{w \in \mathcal{F}_w}\sum_{i=1}^{n} \prns{\prns{r_i-q(s_i,a_i)+\max_{a'\in \Acal} q(s'_i,a')}w(s_i,a_i)-\zeta w^2(s_i,a_i)}.
\end{align*}
Note MSBO was originally proposed using $\zeta=0.5$. We consider general $\zeta>0$. 

Here, given a feature map $\phi:\S\times\A\in\R d$ with $\|\phi(s,a)\|\leq 1$ we consider
\begin{align*}
\Fcal &=\{\theta^{\top}\phi(s,a):\|\theta^{\top}\phi(s,a)\|_{\infty}\leq M', \theta \in \mathbb{R}^d\},\\
    \Fcal_w &=\{\theta^{\top}\phi(s,a):\|\theta^{\top}\phi(s,a)\|_{\infty}\leq M', \theta \in \mathbb{R}^d\}.
\end{align*}

\begin{theorem}\label{thm:msbo}
Suppose $|r_t|\leq M$ for $t=1,2,\ldots$ and $M'=(1-\gamma)^{-1}M$. Moreover, assume $Q^{*}\in \Fcal$ (realizability), $(\Tcal-I)\Fcal\subset \Fcal_w$ (modified completeness), and (modified feature coverage) 
\begin{align*}
    \lambda_{\min}(\E_{(s,a)\sim \mu_b}[(\Tcal -I)\phi(s,a)\{(\Tcal-I)\phi(s,a)\}^{\top}])\geq \lambda'_0. 
\end{align*}
Then, there exists a universal constant $c>0$, such that,
letting \break$a_n=c(\sqrt{d}+M'\sqrt{M'^2/\zeta+M'+\zeta+1}/\lambda'_0)\sqrt{\log n/n}$, we have for all $\delta\geq a_n$,
\begin{align*}
     \mathbb{P}( \|\hat f-Q^{*}\|_{\infty}\geq \delta)\leq \exp\prns{-\delta^2/a^2_n},
\end{align*}
\end{theorem}

The assumptions are similar to FQI in the sense that we need realizability, completeness, and feature coverage, but the latter two are slightly different. As in \cref{sec:fqi}, by combining our results (\cref{thm: fast rate,cor:finitemargin,cor:infinitemargin,lemma:linearmargin,lemma:tabularmargin}) with \cref{thm:msbo}, we obtain a fast regret rate for MSBO. Specifically, if $\alpha<\infty$, we obtain a regret rate of $(\log n/n)^{-(\alpha+1)/2}$, which is faster than the rate of the $O(1/\sqrt{n})$ rate in the analysis of \citet{ChenJinglin2019ICiB}.

{\blockedit 
\section{Extension to High Probability and $L_p$-norm Modes of Convergence}\label{sec:lp}

The fast rates discussed so far require point-wise convergence (\cref{cond:ptwse}), which may be strong -- for example, in \cref{sec:fqi} we establish it for FQI with linear functions by actually proving $L_\infty$-convergence, which is stronger than mean-squared-error (MSE) convergence (and also than point-wise convergence, which is incomparable to MSE convergence).  However, when working with general function classes instead of linear models, for example classes described only by metric-entropy conditions, we may only get MSE convergence. Also, our guarantees were on average over data. In this section, we present (slower) fast rates that work with general function classes and hold with high probability.

First, we present an alternative to theorem \cref{thm: fast rate} where we replace \cref{cond:ptwse} with $L_p$-convergence.
For $p\geq1$, a function $g(s,a)$, and an $(s,a)$-distribution $\mu$, define the $L_p$-norm $\|g\|_{p,\mu}^p=\expect_{s,a\sim\mu}\abs{g(s,a)}^p$. For $p=\infty$, we set $\|g\|_{p,\mu}=\|g\|_\infty$ independent of $\mu$.

\begin{theorem}\label{thm: fast rate Lp}
Let $\bar f$ be given.
Suppose \cref{ass: margin} holds. Then, for $\bar\pi=\pi_{\bar f}$, we have for any $p\in[1,\infty]$,
\begin{align*}
V^*-V^{\bar\pi}&\leq \frac2{1-\gamma}
\frac
1%
{\delta_0^{(p-1)\alpha/(p+\alpha)}}
(\|Q^*-\bar f\|_{p,d^{\bar\pi}\times\bar\pi}+\|Q^*-\bar f\|_{p,d^{\bar\pi}\times\pi^*})^{p(1+\alpha)/(p+\alpha)}.
\end{align*}
\end{theorem}

Note this statement holds deterministically for each $\bar f$; data does not make an appearance. 
This is in contrast to \cref{thm: fast rate}, which holds on average over the data. This means that \emph{any} guarantee on $\sup_{\pi\in\Pi,\pi'}\|Q^*-\hat f\|_{p,d^{\pi}\times\pi'}$ for a data-driven $\hat f$ -- whether a bound on its expectation or a high-probability bound -- can be directly translated into an analogous bound on $V^*-V^{\hat\pi}$. 
For example, using $p=\infty$, we can combine \cref{thm: fast rate Lp} with \cref{lemma: l infty fqi} to obtain a high-probability $O(1/n)$ regret guarantee for linear FQI. More interestingly, focusing on $p=2$ (MSE), \cref{thm: fast rate Lp} allows us to extend to general function approximation, beyond just linear functions. In particular, in the next section we explain how we can obtain high-probability bounds on MSE ($p=2$) for FQI with general function approximation, and we compare the resulting rates to those obtained in the previous sections.

\subsection{MSE Convergence of FQI with General Function Approximation}\label{subsec:MSE_FQI}

The $L_2$-convergence of FQI and other $Q^*$-estimators have been extensively investigated \citep[\eg,][among others]{munos2003error,munos2005error,agarwal2019reinforcement,uehara2023refined}. For the sake of completeness, we include a representative result. The following result is an adaptation of the arguments in the proof of lemma 4.4 in \citet{agarwal2019reinforcement} combined with results for nonparametric least squares in order to accommodate infinite function classes. We use a stronger form of \cref{ass: min eigen} apt for general function approximation, not just linear functions \citep{ChenJinglin2019ICiB,munos2005error}. And, we constrain the complexity of $\Fcal$ using its critical radius \citep{wainwright2019high}. 

\begin{assumption}[Concentratability coefficient] \label{cond:conc}
There exists a constant $C_{\mathrm{all}}(\Pi)>0$ such that for any $\pi\in\Pi$, any $t$, and any arbitrary sequence of policies $\pi_0,\dots,\pi_t$, the marginal distribution of $s_t,a_t$ under $s_0\sim d^\pi$, $a_0\sim\pi_0(\cdot\mid s_0)$, $s_1\sim P_s(\cdot\mid s_0,a_0)$, $a_1\sim\pi_1(\cdot\mid s_1)$, $s_2\sim P_s(\cdot\mid s_1,a_1)$, $a_2\sim\pi_2(\cdot\mid s_2)$, \ldots, $a_t\sim\pi_t(\cdot\mid s_t)$ is 
absolutely continuous with respect to $\mu_b$ with Radon-Nykodim derivative bounded by $C_{\mathrm{all}}(\Pi)$. 
\end{assumption}

\begin{assumption}[Critical radius] \label{cond:critrad}
The function class $\Fcal$ is symmetric ($f\in\Fcal\implies -f\in\Fcal$) and $M$-bounded ($f\in\Fcal\implies\|f\|_\infty\leq M$) and $\rho_n$ bounds a solution $\delta$ to the inequality
$$
\frac1{2^n}\sum_{\epsilon\in\{-1,1\}^n}\sup_{f \in \mathrm{conv}(\mathcal{F}),\,\|f\|_{2,\mu_b} \leq \delta }  \left|\frac1n \sum_{i=1}^n \epsilon_i f(s_i,a_i) \right| \leq \frac{\delta}{M}.
$$
\end{assumption}

\begin{lemma}\label{lem:FQI}
Suppose $|r_t|\leq M$ for $t=1,2,\ldots$
and that \cref{ass: completeness,cond:conc,cond:critrad} hold. 
Then, with probability at least $1-\delta$, we have that 
\begin{align*}
    \sup_{\pi\in\Pi,\pi'}\|Q^*-\hat f_K \|_{2,d^\pi\times \pi'}\leq  \frac{1}{(1-\gamma)^{2}}M\sqrt{C_{\mathrm{all}}(\Pi)} \left(c_1\rho_n +c_2\sqrt{\log(K/\delta)/n} \right)+\frac{\gamma^K}{1-\gamma}M,
\end{align*}
where $c_1,c_2$ are universal constants. 
\end{lemma}

For the simplest of function classes (\eg, linear, finite, VC-subgraph, \etc), we have $\rho_n=O(1/\sqrt{n})$ with high probability \citet{wainwright2019high}.
In these cases, the MSE is bounded by $O(1/\sqrt{n})$ in high probability. Combining with \cref{thm: fast rate Lp} with $\alpha=1$ leads to $O(1/n^{2/3})$-regret with high probability. In comparison, for linear classes, leveraging the $O(1/\sqrt{n})$-rate pointwise convergence we prove in \cref{sec:fqi}, \cref{thm: fast rate} with $\alpha=1$ yields a faster $O(1/n)$-regret on average over the data $\mathcal{D}$. \cref{thm: fast rate Lp} also accommodates more complex and nonparametric function classes; we refer the reader to \citet{wainwright2019high} to bounds on critical radii for various function classes.

\subsection{Guarantees Under Single-Policy Concentratability and Single-Policy Margin}

The fast-rate results established so far, including in this section, require that the offline data covers the state-action potentially distribution induced by \emph{any} policy (\cref{cond:conc} or \cref{ass: min eigen}, depending on the result). This can sometimes be restrictive in offline RL since exploration is limited. We now consider a relaxation where we only require coverage and margin with respect to the distribution induced by the optimal policy. We do this, however, at the cost of obtaining a slower rate than before. Nonetheless, the rate can still be arbitrarily fast as $p\infty,\alpha\to\infty$.

We now adapt \cref{ass: margin,cond:conc} to be formulated exclusively in terms of the optimal policy without considering any other policies, and we then prove an analog to \cref{thm: fast rate Lp}, translating the $L_p$-error of any $\bar f$ to a regret guarantee (deterministically, \ie, surely over the data).

\begin{assumption}[Single-Policy Margin] \label{ass: margin_single}
There exist constants $\delta_0 >0, \alpha\in[0,\infty]$ such that for all $\delta>0$,
\begin{align*}
\pr_{s\sim d^{\pi^*}}(0<\Delta(s) \le \delta) \le (\delta/\delta_0)^{\alpha},
\end{align*}
\end{assumption}

\begin{assumption}[Single-Policy Concentratability] \label{cond:conc_single}
Let $\pi^\mathrm{Uni}$ denote the uniform policy over $\Acal$. There exists a constant $C_{\mathrm{sin}}>0$ such that  $d^{\pi^*}\times \pi^\mathrm{Uni}$ and $d^{\pi^*}\times\pi^*$ are absolutely continuous with respect to $\mu_b$ with Radon-Nykodim derivatives bounded by $C_{\mathrm{sin}}$.
\end{assumption}

\begin{theorem}\label{thm: fast rate Lp2_single}
Let $\bar f$ be given. Suppose $|r_t|\leq M$ for $t=1,2,\ldots$ and that \cref{ass: margin_single,cond:conc_single} hold. Then, for $\bar\pi=\pi_{\bar f}$, we have 
$$
V^*-V^{\bar\pi}\leq  2(1-\gamma)^{-2}M|\Acal| (2/\delta_0)^{\frac{p \alpha}{\alpha + p}}   C^{\frac{\alpha}{\alpha +p}}_{\mathrm{sin}} \|Q^{\star}-\bar f \|^{\frac{p\alpha}{(\alpha +p)}}_{p,\nu}. 
$$
\end{theorem}

The next question is how to obtain guarantees for
$\|Q^{\star}-\bar f \|_{p,\nu}$ \emph{without} coverage with respect to all policies, as is needed in \cref{lem:FQI}. 
Recently, \citet{uehara2023refined} demonstrated that using a minimax-type estimator $\hat f$ we can obtain guarantees for $\|Q^{\star}-\hat f \|_{2,\nu}$ under just single-policy concentrability and the realizability of a Lagrange multiplier. Together with \cref{thm: fast rate Lp2_single} this would lead to a regret guarantee under single-policy concentrability and single-policy margin conditions. However, the rate can be much slower, as the exponent on the $L_p$-error changes from $p(1+\alpha)/(\alpha+p)$ in \cref{thm: fast rate Lp} to $(p\alpha)/(\alpha+p)$ in \cref{thm: fast rate Lp2_single}. For example, even if we can guarantee $\|Q^*-\bar f \|_{2}=O_p(1/\sqrt{n})$ as we can for linear models then the resulting rate obtained from \cref{thm: fast rate Lp2_single} would be only $O_p(1/n^{1/3})$ when $\alpha=1$.}

\section{Related Literature}

\paragraph{Tabular offline RL} \citet{pmlr-v125-agarwal20b,edssjs.BD18D58A20130101} analyze the tabular case where we have a generative model \citep{pmlr-v125-agarwal20b,edssjs.BD18D58A20130101}, that is an oracle for drawing from the MDP's reward and transition distributions, but assuming a generative model is much stronger than our offline setting, where we just see data passively. In the tabular offline setting, the minimax optimal regret rate is obtained by \citet{yin2021near} and is $O(1/\sqrt{n})$. Our result in the tabular case is $\exp(-\Omega(n))$, but it depends on instance parameters such as $\delta_0$, \ie, it is \emph{not} minimax if $\delta_0$ is allowed to vary, and in particular approach $0$. 

\paragraph{Value-based offline RL}  Value-based offline RL is an approach to offline RL where we estimate $Q^{*}$ and then use the corresponding greedy (\ie, argmax) policy (some also consider a smoothed softmax version). This is the approach we studied here. The most common way to estimate $Q^{*}$ is FQI \citep{ernst2005tree}, which we analyze in \cref{sec:fqi}. \citet{munos2008finite,FanJianqing2019ATAo,ChenJinglin2019ICiB} have analyzed finite-sample error bounds for FQI. Since they obtain bounds on the \emph{average} error, their analysis is not directly applicable to our setting. We therefore established uniform convergence in order to derive fast regret rates for the resulting greedy policy.  

Another common method to estimate $Q^*$ is modified BRM and its variants, which we analyze in \cref{sec:sbeed}. The finite-sample error bound of the $Q^{*}$-function has been analyzed in \citet{ChenJinglin2019ICiB,XieTengyang2020QASf} when the hypothesis class is finite. In a general function approximation setting (such as linear), a slow rate of $O({n}^{-1/4})$ is obtained by \citet{antos2008learning} for policy evaluation (estimate $Q^\pi$ for a given $\pi$). Since existing analyses are for average errors and/or not tight, it is not directly applicable. We therefore established uniform convergence. 

Beyond FQI and BRM/MSBO, there are many offline estimators for $Q^{*}$ such as SBEED \citep{pmlr-v80-dai18c}, which uses a smoothed version of the max operator in modified RBM, and MABO \citep{XieTengyang2020QASf}, which is derived by a conditional moment equation formulation of $Q^{*}$. In all of the aforementioned value-based offline RL work including these two, the final regret is $O(1/\sqrt{n})$. Our analysis shows that we can obtain the faster rate depending on the margin condition. 

\paragraph{Policy-based offline RL} Policy-based offline RL is an approach where we directly optimize a policy among some restricted policy class. One common approach is fitted policy iteration \citep{LagoudakisMichail2004LPI,LazaricAlessandro2010FAoL}. Finite-sample regret bound have been analyzed by \citet{antos2008learning,liu2020provably}, which show that the final regret is $O(1/\sqrt{n})$. 
Another common way is offline policy gradient \citep{NachumOfir2019APGf}. \citet{EEOPG2020} showed the asymptotic regret of a offline policy gradient method based on efficient estimation is $O(1/\sqrt{n})$. Note our analysis here does not apply to the policy-based approach. We leave it as future work. 

\paragraph{Offline policy evaluation} Offline policy evaluation (OPE) is the task evaluating the policy value of a single policy \citep{Precup2000}, \ie, estimating $V^\pi$ for a given $\pi$. The typical error rate is $O(1/\sqrt{n})$ \citep{DuanYaqi2020MOEw,thomas2016,kallus2022efficiently,liu2018breaking}. \citet{kallus2022efficiently,kallus2020double} focuses on how to reduce the constant in the leading $1/\sqrt{n}$ term. Note that our work does \emph{not} imply a fast rate for OPE since regret is different from estimation error. In particular, our fast rate leverages the impact of the downstream decision-making problem, after estimation.

\paragraph{Margin Condition and Fast Rate}

In classification, \cite{tsybakov2004optimal,audibert2007fast} showed that both empirical risk minimization and plug-in methods can achieve $o(1/\sqrt{n})$ fast rates under a margin condition that quantifies the concentration of mass of $P(Y=1\mid X)$ near $1/2$.
An analogous condition has been used in contextual
bandits to quantify the separation between arms and get low regret \citep{goldenshluger2013linear,perchet2013multi,hu2022smooth,bastani2020online}.
In particular, such margin conditions for $\alpha<\infty$ can be much weaker than assuming strict separation of arms ($\alpha=\infty$), which is often unrealistic.
\citet{luedtke2020performance} consider fast rates for offline contextual bandits, showing sub-square-root rates for policy-based methods and leveraging a bandit-type margin condition for analyzing value-based methods.
\citet{hu2022fast} use a margin condition that characterizes the  distribution of near-degeneracy in contextual linear optimization problems and obtained fast regret rates in that problem for both plug-in policies and empirical risk minimization.
\edit{The key difference to \citet{luedtke2020performance,audibert2007fast,hu2022fast} is they focus on one-shot decisions with no horizon -- classification of a binary label, classification of which of two actions is optimal, or classification of which vertex of a polytope is optimal -- whereas we consider offline reinforcement learning where handling state-transition dynamics is key both in obtaining regret rates or pointwise learning rates on $Q$-functions.}

\paragraph{Gap Assumption in Online RL} A variety of work on \emph{online} RL makes use of a strict gap assumption \citep{du2020good,du2019provably,simchowitz2019non,yang2021q}, wherein there is a strictly positive lower bound on the reward gap between any two actions, akin to $\alpha=\infty$ in our \cref{ass: margin}. Under such an assumption, instance-dependent logarithmic regret bounds are obtained for the online problem, which is in agreement with exponential decay of regret for the offline problem. Nonetheless, a strict gap may not exist when $\S$ is continuous or it can lead to misleading asymptotic predictions even when $\S$ is finite but very large. Our results in the offline problem highlight a finer spectrum of margin behavior, which possibly suggest an avenue of future work on the online problem to obtain sub-square-root, super-polylogarithmic instance-dependent regret bounds as seen in contextual bandits \citep{hu2022smooth}.

\section{Proofs}

\subsection{Proofs for \cref{sec:q greedy}}

\begin{myproof}[Proof of \cref{thm: fast rate}]
For any policy $\pi$, define $A^{\pi}(s,a) = Q^{\pi} (s,a) - V^{\pi}(s)$, and $A^*(s,a) = A^{\pi^*}(s,a) \le 0$.
By the performance difference lemma (\citet[Lemma 1.16]{agarwal2019reinforcement}),
\begin{align}
    (1-\gamma) (V^* - V^{\hat{\pi}}) = & \expect_{s\sim d^{\hat{\pi}}} [-A^*(s, \hat{\pi}(s)] \notag \\
    = & \expect_{s\sim d^{\hat{\pi}}} [Q^*(s, \pi^*(s))-Q^*(s, \hat{\pi}(s))]. \label{eq: performance difference}
\end{align}

Define the events
\begin{align*}
    & B_0 =\{0 < Q^*(s, \pi^*(s)) - Q^*(s,\hat{\pi}(s)) \le 2a_n\}, \\
    & B_i =\{ 2^ia_n < Q^*(s, \pi^*(s)) - Q^*(s,\hat{\pi}(s)) \le 2^{i+1}a_n\} \quad i\ge 1,\\
    & B' =\{ 2^{i_{\max}+1}a_n < Q^*(s, \pi^*(s)) - Q^*(s,\hat{\pi}(s))\}.
\end{align*}
Peeling on $Q^*(s, \pi^*(s)) - Q^*(s,\hat{\pi}(s))$ we get
\begin{align}
   & \expect_{s\sim d^{\hat{\pi}}} [Q^*(s, \pi^*(s))-Q^*(s, \hat{\pi}(s))]\notag\\
    &=  \expect_{s\sim d^{\hat{\pi}}} [\sum_{i=0}^{i_{\max}}( Q^*(s, \pi^*(s)) - Q^*(s,\hat{\pi}(s)) )\ind\{ B_i\}+( Q^*(s, \pi^*(s)) - Q^*(s,\hat{\pi}(s)) )\ind\{ B'\}] \notag\\
    &\le  2a_n \sum_{i=0}^{i_{\max}}2^i \pr_{s\sim d^{\hat{\pi}}}( B_i) + Q_{\max} \pr_{s\sim d^{\hat{\pi}}}( B'). \label{eq: peeling}
\end{align}
In what follows, we control $\expect_\D \expect_{s\sim d^{\hat{\pi}}}[\ind\{B_i\}]$ for $i\in [0, i_{\max}]$ and $\expect_\D \expect_{s\sim d^{\hat{\pi}}}[\ind\{B'\}]$, which, combined with \cref{eq: performance difference,eq: peeling}, would give an upper bound on $\expect_\D[V^* - V^{\hat{\pi}}]$.

First of all, by \cref{ass: margin},
\begin{align*}
    \expect_\D \expect_{s\sim d^{\hat{\pi}}}[\ind\{B_0\}] \le & \sup_{\pi}\pr_{s\sim d^{\pi}}(0<\Delta(s) \le 2a_n)\\
    \le & 2^{\alpha} a_n^{\alpha}/\delta_0^{\alpha}.
\end{align*}
Second, for $i\ge 1$,
\begin{align*}
     \ind\{ B_i\}
     = & \ind\{\hat{f}(s,\hat{\pi}(s)) \ge \hat{f}(s,\pi^*(s)), 2^ia_n < Q^*(s, \pi^*(s)) - Q^*(s,\hat{\pi}(s)) \le 2^{i+1}a_n  \}\\
     \le & \ind\{ Q^*(s, \pi^*(s))-\hat{f}(s,\pi^*(s))  + \hat{f}(s,\hat{\pi}(s))- Q^*(s,\hat{\pi}(s)) - 2^i a_n >0, \\
     & 0< Q^*(s, \pi^*(s)) - Q^*(s,\hat{\pi}(s)) \le 2^{i+1}a_n \}\\
     \le & \ind\{ Q^*(s, \pi^*(s))-\hat{f}(s,\pi^*(s))> 2^{i-1}a_n, 0<Q^*(s, \pi^*(s)) - Q^*(s,\hat{\pi}(s)) \le 2^{i+1}a_n \} \\
      & + \ind\{\hat{f}(s,\hat{\pi}(s))- Q^*(s,\hat{\pi}(s)) > 2^{i-1}a_n, 0<Q^*(s, \pi^*(s)) - Q^*(s,\hat{\pi}(s)) \le 2^{i+1}a_n  \}\\
      \le & \ind\{ Q^*(s, \pi^*(s))-\hat{f}(s,\pi^*(s))> 2^{i-1}a_n, 0<\Delta(s) \le 2^{i+1}a_n \} \\
      & + \ind\{\hat{f}(s,\hat{\pi}(s))- Q^*(s,\hat{\pi}(s)) > 2^{i-1}a_n, 0<\Delta(s) \le 2^{i+1}a_n  \}.
\end{align*}
where the second inequality comes from a union bound, and the last one comes from the definition of $\Delta$.
Therefore, we have for $i\ge 1$,
\begin{align*}
     &\expect_\D \expect_{s\sim d^{\hat{\pi}}}[\ind\{ B_i\}] \\
     \le & \expect_\D \expect_{s\sim d^{\hat{\pi}}}[\ind\{Q^*(s, \pi^*(s))-\hat{f}(s,\pi^*(s))> 2^{i-1}a_n, 0<\Delta(s)\le 2^{i+1}a_n\}] \\
    & + \expect_\D \expect_{s\sim d^{\hat{\pi}}}[\ind\{\hat{f}(s,\hat{\pi}(s))- Q^*(s,\hat{\pi}(s)) > 2^{i-1}a_n, 0<\Delta(s) \le 2^{i+1}a_n\}]\\
     \le & \sup_{\pi}\expect_\D \expect_{s\sim d^{\pi}}[\ind\{Q^*(s, \pi^*(s))-\hat{f}(s,\pi^*(s))> 2^{i-1}a_n, 0<\Delta(s) \le 2^{i+1}a_n\}] \\
    & + \sup_{\pi} \expect_\D \expect_{s\sim d^{\pi}}[\ind\{\hat{f}(s,\pi(s))- Q^*(s,\pi(s)) > 2^{i-1}a_n, 0<\Delta(s) \le 2^{i+1}a_n\}]\\
      \le & \sup_{\pi}\expect_{s\sim d^{\pi}}[\ind\{ 0<\Delta(s) \le 2^{i+1}a_n\} \pr_{\mathcal{D}}(Q^*(s, \pi^*(s))-\hat{f}(s,\pi^*(s))> 2^{i-1}a_n)] \\
    & + \sup_{\pi} \expect_{s\sim d^{\pi}}[\ind\{ 0<\Delta(s) \le 2^{i+1}a_n\} \pr_{\mathcal{D}}(\hat{f}(s,\pi(s))- Q^*(s,\pi(s)) > 2^{i-1}a_n)].
\end{align*}
Then from \cref{eq: infty convergence,ass: margin} we have for $i\in [1,i_{\max}]$,
\begin{align*}
     \expect_\D \expect_{s\sim d^{\hat{\pi}}}[\ind\{B_i\}] 
    \le & 2 C \exp(-  2^{2i-2}) \sup_{\pi\in\Pi}\pr_{s\sim d^{\pi}}(0<\Delta(s) \le 2^{i+1}a_n) \\
    \le & 2 C \exp(- 2^{2i-2}) (2^{i+1}a_n/\delta_0)^{\alpha}.
\end{align*}
Finally, since
\begin{align*}
    \ind\{ B'\}
     = & \ind\{\hat{f}(s,\hat{\pi}(s)) \ge \hat{f}(s,\pi^*(s)), 2^{i_{\max}+1}a_n < Q^*(s, \pi^*(s)) - Q^*(s,\hat{\pi}(s))  \}\\
     \le & \ind\{ Q^*(s, \pi^*(s))-\hat{f}(s,\pi^*(s))  + \hat{f}(s,\hat{\pi}(s))- Q^*(s,\hat{\pi}(s)) - 2^{i_{\max}+1}a_n >0 \}\\
     \le & \ind\{ Q^*(s, \pi^*(s))-\hat{f}(s,\pi^*(s))> 2^{i_{\max}}a_n\}  + \ind\{\hat{f}(s,\hat{\pi}(s))- Q^*(s,\hat{\pi}(s)) > 2^{i_{\max}}a_n  \},
\end{align*}
by \cref{eq: infty convergence} we have
\begin{align*}
    \expect_\D \expect_{s\sim d^{\hat{\pi}}}[\ind\{B'\}] \le & \sup_{\pi\in\Pi}\expect_{s\sim d^{\pi}}[\pr_{\mathcal{D}}(Q^*(s, \pi^*(s))-\hat{f}(s,\pi^*(s))> 2^{i_{\max}}a_n)] \\ & + \sup_{\pi\in\Pi} \expect_{s\sim d^{\pi}}[\pr_{\mathcal{D}}(\hat{f}(s,\pi(s))- Q^*(s,\pi(s)) > 2^{i_{\max}}a_n)]\\
    \le & 2 C \exp\prns{-  2^{2i_{\max}}}\\
    \le & 2 C \exp(-  \delta_1^2/(4a_n)^2),
\end{align*}
where the last inequality comes from the definition of $i_{\max}$.

Putting all pieces together we get
\begin{align*}
    \expect_\D[V^* - V^{\hat{\pi}}] \le & \frac{2^{\alpha+1} }{(1-\gamma) \delta_0^{\alpha}} \prns{1  + \sum_{i=1}^{i_{\max}}  C \exp\prns{- 2^{2i-2}}2^{(\alpha+1)i+1}} a_n^{\alpha+1} \\
    & + \frac{2Q_{\max}C}{1-\gamma} \exp\prns{-   \delta_1^2/(4a_n)^2}.
\end{align*}
\end{myproof}

\begin{myproof}[Proof of \cref{cor:finitemargin}]
The statement comes from setting $\delta_1 = \infty$ in \cref{thm: fast rate}. We now provide an upper bound on $c(\alpha) = \sum_{i=1}^{\infty}  \exp\prns{- 2^{2i-2}}2^{(\alpha+1)i+1}$.

Define $f(x) = \exp(-2^{2x-2}) 2^{(\alpha+1)x}$. The maximizer of $f(x)$ is $x_0 = \frac{\log(\alpha+1)+\log 2}{2\log 2}$, and $f(x_0) = \prns{\frac{2(\alpha+1)}{e}}^{(\alpha+1)/2}$. Therefore,
\begin{align*}
     \sum_{i=1}^{\infty}  \exp\prns{- 2^{2i-2}}2^{(\alpha+1)i+1} 
    = & 2\sum_{i=1}^{\lfloor x_0\rfloor -1}  f(i) + 2f(\lfloor x_0\rfloor) + 2\sum_{i=\lfloor x_0\rfloor +1}^{\infty}  f(i) \\
    \le & 2 \int_{1}^{\lfloor x_0\rfloor }  f(x) dx+ 2 f(x_0) + 2\int_{\lfloor x_0\rfloor }^{\infty}  f(x) dx \\
    = & 2 \int_{1}^{\infty }  f(x) dx+ 2 f(x_0) \\
    = & \frac{2^{\alpha+1}\Gamma\prns{\frac{\alpha+1}{2},1}}{\log 2} + 2\prns{\frac{2(\alpha+1)}{e}}^{(\alpha+1)/2}.
\end{align*}
\end{myproof}

\begin{myproof}[Proof of \cref{lemma:linearmargin}]
First, for clarity, we provide a very short proof of a weaker result where $\delta_0=\fprns{6\mu_{\max}\sum_{a\in\A}\sum_{a'\in\A:\beta_a\neq\beta_{a'}}\magd{\beta_a-\beta_{a'}}^{-1}}^{-1}$.
For simplicity suppose all $\{\beta_a:a\in\A\}$ are distinct; otherwise we can simply eliminate duplicates.
Letting $V_d(R)$ be the volume of the $R$-radius $d$-ball, we have for any $\pi\in\Pi$,
\begin{align*}
\pr_{s\sim d^\pi}\prns{0<\Delta(s)\leq\delta}&=\pr_{s\sim d^\pi}\prns{\exists a\neq a':0<Q(s,a)-Q(s,a')\leq\delta}\\
&\leq\sum_{a\neq a'}\pr_{s\sim d^\pi}\prns{0<(\beta_a-\beta_{a'})\tr\psi(s)\leq\delta}\\
&\leq \mu_{\max}\sum_{a\neq a'}\int_0^{\delta/\magd{\beta_a-\beta_{a'}}}V_{d-1}((1-u^2)_+^{1/2})du\\
&\leq \mu_{\max}\sum_{a\neq a'}6\delta/\magd{\beta_a-\beta_{a'}},
\end{align*}
since the volume of a unit ball in any dimension is always less than $\frac{8\pi^2}{15}\leq6$.

Now we present an argument to tighten the above so that the inner sum in $\delta_0$ becomes a max. Again suppose all $\{\beta_a:a\in\A\}$ are distinct; else eliminate duplicates.
Letting $\op{Vol}$ denote the Lebesgue measure and $\mathcal B_d=\braces{\magd v\leq1}$ the unit ball, we have
\begin{align*}
&\pr_{s\sim d^\pi}\prns{0<\Delta(s)\leq\delta}\leq\sum_{a'\in\A}\pr_{s\sim d^\pi}\prns{0<\Delta(s)\leq\delta,\,a'\in\argmax_a Q^*(s,a)}\\
&\leq\sum_{a'\in\A}\pr_{s\sim d^\pi}\prns{\forall a\neq a':(\beta_{a'}-\beta_{a})\tr\psi(s)\geq 0,~\exists a\neq a':(\beta_{a'}-\beta_{a})\tr\psi(s)\leq \delta}\\
&\leq \mu_{\max}\sum_{a'\in\A}
\op{Vol}\prns{\bigcup_{a\neq a'}\prns{\mathcal B_d\cap\bigcap_{a''\in\A}\braces{(\beta_{a'}-\beta_{a''})\tr v\geq0}\cap\braces{(\beta_{a'}-\beta_{a})\tr v\leq \delta}}}\\
&\leq \mu_{\max}\sum_{a'\in\A}
\op{Vol}\prns{\bigcup_{a\neq a'}\prns{\mathcal B_d\cap\bigcap_{a''\neq a}\braces{\bar\beta_{a',a''}\tr v\geq0}\cap\braces{\bar\beta_{a',a}\tr v\leq \delta/\min_{a\neq a'}\magd{\beta_{a'}-\beta_a}}}},
\end{align*}
where $\bar\beta_{a',a}=\frac{\beta_{a'}-\beta_a}{\magd{\beta_{a'}-\beta_a}}$. The first inequality is by union bound, the second by definition of $\Delta$ and including $0$ in the event, the third by the uniform upper bound on $d^\pi$, and the fourth by inclusion as we are only increasing the half space in the last term for each $a',a$.
We will next show that the inner volume term is upper bounded by $6\delta/\min_{a\neq a'}\magd{\beta_{a'}-\beta_a}$, yielding the result.

We now study the inner volume term. To abstract things, consider $\beta_1,\dots,\beta_m$ with $\magd{\beta_i}=1$ and the polyhedral cones $K^{(k)}=\{v:\beta_i\tr v\geq0\,\forall i=1,\dots,k\}$ for every $k=1,\dots,m$.
We are then concerned with
$$V=\op{Vol}\prns{\bigcup_{i=1}^m\prns{\mathcal B_d\cap K^{(m)}\cap\braces{\beta_i\tr v\leq \delta'}}}.$$
Placing a prism of height $\delta'$ (in the direction of $\beta_i$) on top of $\mathcal B_d\cap K^{(m)}\cap\braces{\beta_i\tr v= 0}$ for each $i$, we see that the sum of the prims' volumes upper bounds $V$: we are only overcounting volume outside the sphere and any overlaps between the prisms placed at different faces. Let $\partial \mathcal B_d=\{\magd v=1\}$ be the unit sphere shell
and let $\rho=d-1$ be the proportionality between the volume inside the $(d-1)$-dimensional unit sphere and its area.
Notice that the sum of the prisms' volume is equal to $\delta'$ times $\rho^{-1}$ times the perimeter of the spherical polyhedron that $K^{(m)}$ defines on the unit sphere, that is, $\abs{\partial K^{(m)}\cap \partial \mathcal B_d}$.
We claim that $\abs{\partial K^{(m)}\cap \partial \mathcal B_d}\leq \abs{\partial K^{(m-1)}\cap \partial \mathcal B_d}$.
If $\beta_m$ does not intersect the interior of $K^{(m)}$ then this is trivial. Suppose it does intersect. Let $H=\{\beta_m\tr v\geq 0\}$ and $H'=\{\beta_m\tr v\leq 0\}$. Then $\abs{K^{(m-1)}\cap \partial H\cap \partial \mathcal B_d}\leq \abs{\partial K^{(m-1)}\cap H'\cap \partial \mathcal B_d}$ because if we project $\partial K^{(m-1)}\cap H'\cap \partial \mathcal B_d$ onto $\partial H\cap \partial \mathcal B_d$ then we obtain $K^{(m-1)}\cap \partial H\cap \partial \mathcal B_d$ (that is, one side of a spherical polyhedron cannot be larger than the sum of the other sides). Therefore, by adding $\beta_m$ to $K^{(m-1)}$ to obtain $K^{(m)}$, we have lost more perimeter than we have gained, as was claimed. By repeating this, we obtain $\abs{\partial K^{(m)}\cap \partial \mathcal B_d}\leq \abs{\partial K^{(1)}\cap \partial \mathcal B_d}$. Next, notice that $\abs{\partial K^{(1)}\cap \partial \mathcal B_d}/\rho$ is just the volume of the $(d-1)$-dimensional unit ball, which is $\frac{\pi^{(d-1)/2}}{\Gamma((d+1)/2)}\leq 6$. Hence, $V\leq 6\delta'$.
\end{myproof}

{\blockedit

\begin{myproof}[Proof of \cref{lemma: margin for continuous functions}]
    For any $\pi\in \Pi$, we have
    \begin{align*}
    \pr_{s\sim d^\pi}\prns{0<\Delta(s)\leq\delta}&=\pr_{s\sim d^\pi}\prns{\exists a\neq a':0<Q^*(s,a)-Q^*(s,a')\leq\delta}\\
    &\leq\sum_{a\neq a'}\pr_{s\sim d^\pi}\prns{0<Q^*(s,a)-Q^*(s,a')\leq\delta}\\
    & \le \mu_{\max}\sum_{a\neq a'}\op{Vol}\prns{\braces{s: 0<Q^*(s,a)-Q^*(s,a')\leq\delta}}.
    \end{align*}

We now bound the inner volume term $\op{Vol}\prns{\braces{s: 0<Q^*(s,a)-Q^*(s,a')\leq\delta}}$.
Without loss of generality, assume for any $s = (s_1, \dots, s_d)$, $\abs{\partial (Q^*(s,a)-Q^*(s,a'))/\partial s_1} \ge \tau_0$.
Fix any $s_2, \dots, s_d$, and view $g(s_1) = Q^*((s_1,s_2, \dots, s_d),a)-Q^*((s_1,s_2, \dots, s_d),a')$ as a function of $s_1$.
Since $\abs{\partial (Q^*(s,a)-Q^*(s,a'))/\partial s_1} \ge \tau_0$, by Darboux's theorem, the partial derivative is always positive or always negative, 
so $g(s_1)$ has at most one zero point, and
\[
    \int_{-1}^1 \ind\braces{g(s_1) \le \delta} \le \delta/\tau_0.
\]
We can then compute the volume:
\begin{align*}
    & \quad \op{Vol}\prns{\braces{s: 0<Q^*(s,a)-Q^*(s,a')\leq\delta}} \\
    & \le \int_{-1}^1 d s_d\dots\int_{-\sqrt{1-\sum_{i=2}^d s_i^2}}^{\sqrt{1-\sum_{i=2}^d s_i^2}} \ind\braces{Q^*((s_1,s_2, \dots, s_d),a)-Q^*((s_1,s_2, \dots, s_d),a') \le \delta } d s_1   \\
    & \le \frac{\delta}{\tau_0} \int_{-1}^1 d s_d\dots\int_{-\sqrt{1-\sum_{i=3}^d s_i^2}}^{\sqrt{1-\sum_{i=3}^d s_i^2}}  d s_2  \\
    & =  \frac{\delta}{\tau_0} \op{Vol}\prns{\mathcal{B}_{d-1}}\\
    & \le \frac{6\delta}{\tau_0},
\end{align*}
and our conclusion follows.
\end{myproof}
}

\subsection{Proofs for \cref{sec:fqi}}

\begin{myproof}[Proof of \cref{lemma: convergence of T operator}]
Let $y_i^f = r_i + \gamma  \max_{a'\in\mathcal{A}} f(s_i',a')$. We can write
\begin{align*}
  y_i^f =  w_f\tr\phi(s_i, a_i) + \epsilon_i^f,  
\end{align*}
where $ \epsilon_i^f = r_i + \gamma  \max_{a'\in\mathcal{A}} f(s_i',a') - \mathcal{T}f (s_i, a_i) $. Note that
$\expect_{r_i, s_i'\sim P_s(\cdot\mid s_i,a_i)} [\epsilon_i^f \mid s_i, a_i]= 0$. 

When the event $\lambda_{\min}(\hat{\Sigma})> n\lambda_0/2 $ holds,
\begin{align*}
    \sup_{f\in \mathcal{F}}\norm{\hat{w}_f' - w_f} =& \sup_{f\in \mathcal{F}}\norm{\hat{\Sigma}^{-1} \prns{\sum_{i=1}^n \phi(s_i, a_i) \epsilon_i^f}} \\
    \le & \sup_{f\in \mathcal{F}}\norm{\hat{\Sigma}^{-1}}_2 \norm{\sum_{i=1}^n \phi(s_i, a_i) \epsilon_i^f}\\
    \le & \frac{2}{n\lambda_0} \sup_{f\in \mathcal{F}}\norm{\sum_{i=1}^n \phi(s_i, a_i) \epsilon_i^f}.
\end{align*}
Therefore, from a union bound we get
\begin{align*}
    \pr\prns{\sup_{f\in \mathcal{F}} \norm{\hat{w}_f' - w_f} \ge \delta } \le  & \pr\prns{\lambda_{\min}(\hat{\Sigma})\le n\lambda_0/2}+\pr\prns{ \sup_{f\in \mathcal{F}}\norm{\sum_{i=1}^n \phi(s_i, a_i) \epsilon_i^f}\ge n\lambda_0\delta/2} \\
    \le & \underbrace{\pr\prns{\lambda_{\min}(\hat{\Sigma})\le  n\lambda_0/2}}_{(\textbf{a})} + \sum_{j=1}^d \underbrace{\pr\prns{\sup_{f\in \mathcal{F}}  \abs{\sum_{i=1}^n \phi_j(s_i, a_i) \epsilon_i^f}\ge  \frac{n\lambda_0\delta}{2\sqrt{d}}}}_{(\textbf{b})}  ,
\end{align*}
where $\phi_j$ is the $j$-th component of $\phi$.
We now aim to bound the two terms on the right hand side.

\paragraph{Bounding (a).}
Note that
\begin{align*}
   &  \lambda_{\max}\prns{\phi(s_i, a_i) \phi(s_i, a_i)\tr} = \max_{\|u\| = 1} u\tr\phi(s_i, a_i) \phi(s_i, a_i)\tr u \le 1,\\
   & \lambda_{\min}\prns{\sum_{i=1}^n \expect\phi(s_i, a_i) \phi(s_i, a_i)\tr} \ge \sum_{i=1}^n\lambda_{\min}( \expect\phi(s_i, a_i) \phi(s_i, a_i)\tr)= n\lambda_0.
\end{align*}
By matrix Chernoff inequality (\citet[Theorem 5.1.1]{tropp2015introduction}),
\begin{align*}
    \pr\prns{\lambda_{\min}(\hat{\Sigma}) \le n\lambda_0/2} \le d\exp(-n\lambda_0/8).
\end{align*}

\paragraph{Bounding (b).}
Let $h^{f}\prns{s,a,r, s'} = \phi_j(s,a)(r + \gamma  \max_{a'\in\mathcal{A}} f(s',a') - \mathcal{T}f (s, a) )$.
We have $ \expect[h^{f}(s,a,r, s')]= 0.$
Define $h_i^{f} = h(s_i,a_i,r_i, s_i')$, $\textbf{h}^{f} = \prns{h_1^{f}, \dots, h_n^{f}}$, and $\textbf{H} = \{\textbf{h}^{f}: f\in \mathcal{F}\}$.
Note that $\sup_{f\in \mathcal{F}} \abs{h^{f}_i} \le M+ B$ for each $i$.

By \citet[Theorem 2.2]{pollard1990empirical}, for any convex increasing $\Phi$,
\begin{align*}
    \expect \Phi\prns{\sup_{f\in \mathcal{F}}  \abs{\sum_{i=1}^n \phi_j(s_i, a_i) \epsilon_i^f}} = & \expect \Phi\prns{\sup_{f\in \mathcal{F}}  \abs{\sum_{i=1}^n h_i^{f}}}\\
    \le & \expect \expect_{\sigma} \Phi\prns{2 \sup_{\textbf{h} \in \textbf{H}}\abs{\langle\sigma, \textbf{h}\rangle}}.
\end{align*}

Let $\Psi(x) = \frac{1}{5} \exp(x^2)$. 
By \citet[Theorem 3.5]{pollard1990empirical},
\begin{align*}
    \expect \expect_{\sigma} \Psi\prns{\frac{1}{J} \sup_{\textbf{h} \in \textbf{H}}|\langle\sigma, \textbf{h}\rangle|} \le 1, \quad \text{where } J = 9\int_0^{\sqrt{n}(M+B)} \sqrt{\log D(\delta, \textbf{H})} d \delta.
\end{align*}
Since
\begin{align*}
    \norm{\textbf{h}^{f} - \textbf{h}^{g}} \le 2\gamma\sqrt{n}\norm{f-g}_{\infty},
\end{align*}
we have
\begin{align*}
    \log D(\delta, \textbf{H}) \le& \log D\prns{\frac{\delta}{2\gamma\sqrt{n}}, \mathcal{F}, \norm{\cdot}_{\infty}}\\
    \le & d\log\prns{1+8\gamma B\sqrt{n}/\delta}.
\end{align*}
where the last inequality comes from \citet[Lemma 5.5 and Example 5.8]{wainwright2019high}.
This implies
\begin{align*}
    J \le & 9\sqrt{nd}(M+B)\int_0^1 \sqrt{\log\prns{1+8/\delta'}} d \delta'\\
    \le &18 \sqrt{nd}(M+B).
\end{align*}
Combining all pieces we get for all $\delta>0$,
\begin{align*}
    \pr\prns{\sup_{f\in \mathcal{F}}  \abs{\sum_{i=1}^n \phi_j(s_i, a_i) \epsilon_i^f} > \delta}\le 5\exp\prns{-\frac{\delta^2}{1296 nd (M+B)^2 }}.
\end{align*}

\paragraph{Bounding the Error $\sup_{f\in\mathcal{F}}\norm{\hat{w}_f\tr\phi - \mathcal{T}f}_{\infty} $.}
Recall that $\hat{w}_f$ is the projection of $\hat{w}_f'$ onto $\mathcal{B}(0,B)$, so we naturally have $\pr\prns{\sup_{f\in \mathcal{F}} \norm{\hat{w}_f - w_f} \ge \delta } = 0$ for $\delta> 2B$. On the other hand, when $\delta\le 2B$,
\begin{align*}
  \pr\prns{\sup_{f\in\mathcal{F}}\norm{\hat{w}_f\tr\phi - \mathcal{T}f}_{\infty} \ge \delta}
  \le & \pr\prns{\sup_{f\in \mathcal{F}} \norm{\hat{w}_f - w_f} \ge \delta } \\
    \le &\pr\prns{\sup_{f\in \mathcal{F}} \norm{\hat{w}_f' - w_f} \ge \delta }\\ 
    \le &  5d\exp\prns{-\frac{n\lambda_0^2\delta^2}{5184 d^2 (M+B)^2 }} +  d\exp\prns{-n\lambda_0/8}\\
    \le  &6d\exp\prns{- \frac{\lambda_0^2}{5184 d^2 (M+B)^2 } n\delta^2}.
\end{align*}
where we used the fact that $\lambda_0 \le 1$.
Our conclusion then follows.
\end{myproof}

\begin{myproof}[Proof of \cref{lemma: l infty fqi}]
Note that for any $k\in[K]$,
\begin{align*}
    & Q^*(s,a) = \mathcal{T} Q^*(s,a) = r(s,a) + \gamma \expect_{s'\sim P(\cdot\mid s,a)} \max_{a'\in\mathcal{A}} Q^*(s',a'),\\
    & \mathcal{T} \hat f_{k-1}(s,a) = r(s,a) + \gamma \expect_{s'\sim P(\cdot\mid s,a)} \max_{a'\in\mathcal{A}} \hat f_{k-1}(s',a'),
\end{align*}
so we have
\begin{align*}
    \|Q^* - \mathcal{T} \hat f_{k-1}\|_{\infty} \le &\sup_{s,a} \gamma\expect_{s'\sim P(\cdot\mid s,a)}\mid \max_{a'\in\mathcal{A}} Q^*(s',a') - \max_{a'\in\mathcal{A}} \hat f_{k-1}(s',a')|\\
    \le & \sup_{s,a} \gamma\expect_{s'\sim P(\cdot\mid s,a)} \max_{a'\in\mathcal{A}}| Q^*(s',a') -  \hat f_{k-1}(s',a')|\\
    \le & \gamma \|Q^*-  \hat f_{k-1}\|_{\infty}.
\end{align*}
Therefore,
\begin{align*}
    \|Q^*- \hat f_k\|_{\infty} \le & \|Q^* - \mathcal{T} \hat f_{k-1}\|_{\infty} + \|\hat f_k - \mathcal{T} \hat f_{k-1}\|_{\infty}\\
    \le & \gamma \|Q^*-  \hat f_{k-1}\|_{\infty} + \|\hat f_k - \mathcal{T} \hat f_{k-1}\|_{\infty}.
\end{align*}
We can recursively repeat the same process for $\|Q^*-  \hat f_{k-1}\|_{\infty}$ till $k=0$, and get 
\begin{align*}
    \|Q^*- \hat f_K\|_{\infty} \le & \sum_{t=0}^{K-1} \gamma^t\|\hat f_{K-t} - \mathcal{T} \hat f_{K-t-1}\|_{\infty} + \gamma^K \|Q^*-  \hat f_{K-1}\|_{\infty}\\
    \le & \sum_{t=0}^{K-1} \gamma^t\|\hat f_{K-t} - \mathcal{T} \hat f_{K-t-1}\|_{\infty} +  \frac{\gamma^K M}{1-\gamma}.
\end{align*}

Note that $K\ge \frac{\log(\lambda_0^2 n/(5184 d^2))}{2\log(1/\gamma)}$ implies that $ \frac{\gamma^K M}{(1-\gamma)}\le\frac{72 d (M+B)}{(1-\gamma)\lambda_0 \sqrt{ n}} $.
Moreover, by \cref{lemma: convergence of T operator}, we know that for any $\delta>0$,
\begin{align*}
    \pr \prns{\sum_{t=0}^{k-1} \gamma^t\norm{f_{k-t} - \mathcal{T} f_{k-t-1}}_{\infty} \ge \frac{\delta}{2}}\le & \pr\prns{\sup_{f\in\mathcal{F}}\norm{\hat{w}_f\tr\phi - \mathcal{T}f}_{\infty} \ge \frac{\delta(1-\gamma)}{2} }\\
    \le & %
    6d\exp(-  \frac{(1-\gamma)^2\lambda_0^2}{20736 d^2 (M+B)^2 } n\delta^2 ).
\end{align*}
Therefore, for any 
$\delta \ge \frac{144 d (M+B)}{(1-\gamma)\lambda_0 \sqrt{ n}}$,
\begin{align*}
    \pr(\|Q^*- f_k\|_{\infty}\ge \delta) \le 6d\exp(-  \frac{(1-\gamma)^2\lambda_0^2}{20736 d^2 (M+B)^2 } n\delta^2 ).
\end{align*}
\end{myproof}

\subsection{Proof for \cref{sec:sbeed}}

\begin{myproof}[Proof of \cref{thm:msbo}]
We begin by introducing notation and the notion of the critical radius.
We introduce $\E_n[f(\cdot)]=1/n\sum_{i}f(s_i,a_i,r_i,s'_i)$. In addition, we use $\E[\cdot]$ to present $\E_{\mu_b}[\cdot]$. We define 
\begin{align*}
    \Phi(q;w) &=\E[\{r-q(s,a)+\gamma \max_{a'}q(s',a')\}w(s,a)],\\
  \Phi_n(q;w) &=\E_n[\{r-q(s,a)+\gamma \max_{a'}q(s',a')\}w(s,a)],\\
     \Phi^{\zeta}_n(q;w) &= \Phi_n(q;w)-\zeta\E_n[w^2],\\
         \Phi^{\zeta}(q;w) &=\Phi(q;w)-\zeta\E[w^2]. 
\end{align*}
Let $\eta_n$ be the upper bound of the critical radius of $\Fcal_{w}$ and 
\begin{align*}
    \Gcal_{q}=\{(s,a)\mapsto w(s,a)\{-q(s,a)+\gamma \max_{a'}q(s',a')+Q^{*}(s,a)+\gamma \max_{a'}Q^{*}(s',a')\}  : w\in \Fcal_w,q\in \Fcal\}. 
\end{align*}
From a standard results on linear models, $\eta_n=c\sqrt{d\log n/n}$. Then, from \citet[Theorem 14.1]{wainwright2019high}, with $1-c_0\exp(-c_1n\eta^2/M'^2)$, for any $\eta \geq \eta_n$, we have 
\begin{align}\label{eq:martin_q}
    \forall w(s,a)\in \Fcal_{w},\,|\E_n[w^2]-\E[w^2]|\leq 0.5\E[w^2]+\eta^2
\end{align}
noting $\eta_n$ upper bounds the critical radius of $\Fcal_{w}$. 

\paragraph{Calculation of the upper bound of $    \sup_{w\in \Fcal_{w}}\{ \Phi_n(\hat f;w)- \Phi_n(Q^{*};w)-2\zeta \E_n[w^2] \} $. }

By definition of $\hat f$ and $Q^{*}\in \Fcal_{q}$, we have 
\begin{align}\label{eq:obvious_def}
    \sup_{w\in \Fcal_{w}}   \Phi^{\zeta}_n(\hat f;w) \leq  \sup_{w\in \Fcal_{w}}   \Phi^{\zeta}_n(Q^{*};w) 
\end{align}
From  \citet[Theorem 14.20]{wainwright2019high}, with probability $1-c_0\exp(-c_1n\eta^2/M'^2)$, for any $\eta \geq \eta_n$,  we have 
\begin{align}\label{eq:martin2_q}
    \forall w\in \Fcal_{w}: |  \Phi_n(Q^{*};w) -\Phi(Q^{*};w)   |\leq cC_1\{\eta \E[w^2]^{1/2}+\eta^2\}. 
\end{align}
Here, we use $l(a_1,a_2):=a_1a_2,a_1=w(s,a),a_2=r-q(s,a)+\gamma \max_{a'} q(s',a') $ is $2(1+\gamma)B$-Lipschitz with respect to $a_1$ by defining $C_1=2(1+\gamma)B$, that is, 
\begin{align*}
    |l(a_1,a_2)-l(a'_1,a_2)|\leq C_1|a_1-a'_1|.  
\end{align*}

Thus, 
\begin{align} 
    \sup_{w \in \Fcal_{w} } \Phi^{\zeta}_n(Q^{*};w) &=    \sup_{w \in \Fcal_{w} } \braces{ \Phi_n(Q^{*};w) -\zeta \E_n[w^2]} && \text{Definition}\nonumber \\
&\leq   \sup_{w \in \Fcal_{w} } \braces{ \Phi(Q^{*};w)+cC_1\eta \E[w^2]^{1/2}+ cC_1\eta^2 -\zeta\E_n[w^2]} && \text{From \cref{eq:martin2_q}}\nonumber \\
&\leq    \sup_{w \in \Fcal_{w} } \braces{ \Phi(Q^{*};w)+cC_1\eta \E[w^2]^{1/2}+cC_1\eta^2 -0.5\zeta\E[w^2]+\zeta\eta^2} && \text{From \cref{eq:martin_q}} \nonumber \\
&\leq  \sup_{w \in \Fcal_{w} }  \braces{\Phi(Q^{*};w)+(4c^2C^2_1/\zeta+cC_1+\zeta) \eta^2}.  \label{eq:help3}
\end{align}
In the last line, we use a general inequality, $a,b>0$:
\begin{align*}
    \sup_{w\in \Fcal_{w} }(a\E[w^2]^{1/2}-b\E[w^2])\leq a^2/4b. 
\end{align*}
Moreover,
{\small 
\begin{align*}
    \sup_{w\in \Fcal_{w}}\{ \Phi^{\zeta}_n(\hat f;w) \}&=     \sup_{w\in \Fcal_{w}}\{ \Phi_n(\hat f;w)- \Phi_n(Q^{*};w)+ \Phi_n(Q^{*};w) -\zeta \E_n[w^2]\}\\
    &\geq     \sup_{w\in \Fcal_{w}}\{ \Phi_n(\hat f;w)- \Phi_n(Q^{*};w) - 2\zeta \E_n[w^2]\}+  \inf_{w\in \Fcal_{w}}\{ \Phi_n(Q^{*};w) +\zeta  \E_n[w^2]\} \\
     &=     \sup_{w\in \Fcal_{w}}\{ \Phi_n(\hat f;w)- \Phi_n(Q^{*};w) - 2\zeta \E_n[w^2]\}+  \inf_{-w\in \Fcal_{w}}\{ \Phi_n(Q^{*};-w) + \zeta \E_n[w^2]\} \\
         &=     \sup_{w\in \Fcal_{w}}\{ \Phi_n(\hat f;w)- \Phi_n(Q^{*};w) - 2\zeta \E_n[w^2]\}+  \inf_{-w\in \Fcal_{w}}\{-\Phi_n(Q^{*};w) + \zeta \E_n[w^2]\} \\
         &=     \sup_{w\in \Fcal_{w}}\{ \Phi_n(\hat f;w)- \Phi_n(Q^{*};w) - 2\zeta \E_n[w^2]\}-  \sup_{-w\in \Fcal_{w}}\{\Phi_n(Q^{*};w)- \zeta \E_n[w^2]\} \\
    & =     \sup_{w\in \Fcal_{w}}\{ \Phi_n(\hat f;w)- \Phi_n(Q^{*};w)-2\zeta \E_n[w^2]\}-\sup_{w\in \Fcal_{w}}\Phi^{\zeta}_n(Q^{*};w).
\end{align*}
}
Here, we use $\Fcal_{w}$ is symmetric.  Therefore, 
\begin{align*}
    \sup_{w\in \Fcal_{w}}\{ \Phi_n(\hat f;w)- \Phi_n(Q^{*};w)-2\zeta \E_n[w^2] \}&\leq \sup_{w\in \Fcal_{w}}\braces{\Phi^{\zeta}_n(Q^{*};w)}+    \sup_{w\in \Fcal_{w}}\{ \Phi^{\zeta}_n(\hat f;w) \}\\
    &\leq 2\sup_{w\in \Fcal_{w}}\Phi^{\zeta}_n(Q^{*};w)\\
    & \leq   \sup_{w \in \Fcal_{w} }  \braces{\Phi(Q^{*};w)+(c^24C^2_1/\zeta+vC_1+\zeta)\eta^2}\\
    &= (c^24C^2_1/\zeta+vC_1+\zeta)\eta^2.  
\end{align*}

\paragraph{Calculation of the lower bound of $    \sup_{w\in \Fcal_{w}}\{ \Phi_n(\hat f;w)- \Phi_n(Q^{*};w)-2\zeta \E_n[w^2] \} $.}
Define 
\begin{align*}
    w_q=(\Tcal-I)q. 
\end{align*}
Suppose $\{\E[w^2_{\hat f} ]\}^{1/2}\geq \eta$, and let $\kappa=\eta/\{2\{\E[w^2_{\hat f} ]\}^{1/2}\}\in [0,0.5]$. Then, noting $\Fcal_{w}$ is star-convex,
\begin{align*}
  \sup_{w\in \Fcal_{w}}\{ \Phi_n(\hat f;w)- \Phi_n(Q^{*};w)-2\zeta \E_n[w^2] \}\geq \kappa\{\Phi_n(\hat f,w_{\hat f})-\Phi_n(Q^{*},w_{\hat f})\}-2\kappa^2\zeta\E_n[w^2_{\hat f}]. 
\end{align*}
since  $\kappa w_{\hat f}\in \Fcal_{w}$. Then,
\begin{align*}
    \kappa^2\E_n[w^2_{\hat f}]& \leq \kappa^2\{1.5\E[w^2_{\hat f}]+0.5\eta^2\}  \tag{ \cref{eq:martin_q}} \\
    & \leq 3\eta^2.  \tag{Definition of $\kappa$}
\end{align*}
Therefore, 
\begin{align*}
     \sup_{w\in \Fcal_{w}}\{ \Phi_n(\hat f;w)- \Phi_n(Q^{*};w)-2\E_n[w^2] \}\geq \kappa\{\Phi_n(\hat f,w_{\hat f})-\Phi_n(Q^{*},w_{\hat f})\}-2\zeta \eta^2. 
\end{align*}
Using 
\begin{align*}
    \Phi_n(q,w_q)-  \Phi_n(Q^{*},w_q)&=\E_n[\{-q(s,a)+Q^{*}(s,a)+\gamma \max_{a'} q(s',a')- \gamma \max_{a'} Q^{*}(s',a')\}w_q(s,a)], 
\end{align*}
from \citet[Theorem 14.20]{wainwright2019high}, with probability $1-c_0\exp(-c_1n\eta^2/M'^2)$, for any $\eta \geq \eta_n$, for any $ q\in \Fcal_{q}$, 
\begin{align*}
    &|\Phi_n(q,w_q)-  \Phi_n(Q^{*},w_q)-\{\Phi(q,w_q)-  \Phi(Q^{*},w_q)\}|\\
    &=|\E_n[\{-q(s,a)+Q^{*}(s,a)+\gamma \max_{a'} q(s',a')- \gamma \max_{a'} Q^{*}(s',a')\}w_q(s,a)  ]\\ &-\E[\{-q(s,a)+Q^{*}(s,a)+\gamma \max_{a'} q(s',a')- \gamma \max_{a'} Q^{*}(s',a')\}w_q(s,a))  ]  |\\ 
    &\leq (\eta\E[\{-q(s,a)+Q^{*}(s,a)+\gamma \max_{a'} q(s',a')- \gamma \max_{a'} Q^{*}(s',a')\}^2w^2_q(s,a)]^{1/2}+\eta^2)\\
 &\leq (\eta M'\{\E[w_q ]\}^{1/2}+\eta^2). 
\end{align*}
Here, we invoke \citet[Theorem 14.20]{wainwright2019high} by treating $l(a_1,a_2)=a_1,\,a_1= \{q(s,a)-Q^{*}(s,a)-\gamma \max_{a'} q(s',a')+ \gamma \max_{a'} Q^{*}(s',a')\}w_q(s,a)\}$. 

Thus, 
\begin{align*}
    \kappa\{\Phi_n(\hat f,w_{\hat f})-  \Phi_n(Q^{*},w_{\hat f})\}& \geq  \kappa\{\Phi(\hat f,w_{\hat f})-  \Phi(Q^{*},w_{\hat f})\}-\kappa(M'\eta\{\E[w^2_{\hat f} ]\}^{1/2}+\eta^2). \tag{$\kappa\leq 0.5$} \\
    & \geq\kappa\{\Phi(\hat f,w_{\hat f})-  \Phi(Q^{*},w_{\hat f})\}-\kappa(M'\eta\{\E[w^2_{\hat f} ]\}^{1/2})-0.5\eta^2 \\
    &\overset{\text{(a)}}= \kappa\E[w_{\hat f}(s,a)(\Tcal-I)\hat f(s,a)]-\kappa(M'\eta\{\E[w^2_{\hat f} ]\}^{1/2})-0.5\eta^2 \\
  &= \frac{\eta}{2\E[w_{\hat f}^2]^{1/2}}\{ \E[w_{\hat f}^2]^{1/2}-M'\eta\{\E[w^2_{\hat f} ]\}^{1/2}\}-0.5\eta^2 \\
    &\geq 0.5\eta \E[\{(\Tcal-I)\hat f\}^2]^{1/2}-(0.5+M')\eta^2. 
\end{align*}
For (a), we use 
\begin{align*}
    \Phi(\hat f,w_{\hat f})-  \Phi(Q^{*},w_{\hat f})&=\E[w_{\hat f}(s,a)\{-\hat f(s,a)+Q^{*}(s,a)+\gamma  \hat\max_{a'} f(s',a')- \gamma \max_{a'} Q^{*}(s',a')\}]\\
    &= \E[w_{\hat f}(s,a)(\Tcal-I)\hat f(s,a)].
\end{align*}

\paragraph{Combining the upper and lower bound of $    \sup_{w\in \Fcal_{w}}\{ \Phi_n(\hat f;w)- \Phi_n(Q^{*};w)-2\zeta \E_n[w^2] \} $.}

Thus, for $\eta>\eta_n$ with probability $1-c_0\exp(-c_1 n\eta^2/M'^2)$, $\{\E[w^2_{\hat f} ]\}^{1/2}\leq \eta$ or 
\begin{align*}
\eta \E[\{(\Tcal-I)\hat f\}^2]^{1/2}-(M'+\zeta )\eta^2\leq c_2\times (M'^2/\zeta+M'+\zeta)\eta^2. 
\end{align*}
Therefore, for $\eta\geq \eta_n$, with $1-c_0\exp(-c_1 n\eta^2/M'^2)$, we have 
\begin{align*}
   \|\hat w-w_{0}\|\lambda'_0\leq \E[\{(\Tcal-I)\hat f\}^2]^{1/2}\leq  c_2(M'^2/\zeta+M'+\zeta+1)\eta. 
\end{align*}
This implies for any $\delta\geq \eta_n$,  we have 
\begin{align*}
    \mathbb{P}( \|\hat f-Q^{*}\|_{\infty}\geq \delta)\leq \exp\prns{-c_1 \frac{n\lambda^2_0\delta ^2}{M'^2\{(M'^2/\zeta+M'+\zeta+1) \}}}. 
\end{align*}
In the end, for $\delta\geq a_n,a_n=c\{\sqrt{d}+M'\sqrt{M'^2/\zeta+M'+\zeta+1}/\lambda_0\}\sqrt{\log n/n}$, 
\begin{align*}
     \mathbb{P}( \|\hat f-Q^{*}\|_{\infty}\geq \delta)\leq \exp\prns{-\delta^2/a_n}. 
\end{align*}

\end{myproof}

\section{Concluding Remarks}

In this paper we established the first sub-$1/\sqrt{n}$ regret guarantees for offline reinforcement learning. In particular, we showed that, given an estimate of $Q^*$, the resulting $Q$-greedy policy has a regret rate given by exponentiating the estimation rate, where the exponent depends on a margin condition. We also showed that quite strong margin conditions generally hold in linear and tabular MDPs, and argued a nontrivial margin should usually hold for a given instance in practice. Our rate-speed-up result relied on pointwise convergence guarantees for $Q^*$ estimates. Since no such exist, we derived new uniform convergence guarantees for FQI and a BRM variant called MSBO (and this implied pointwise convergence). The rates our theory predict are almost exactly what is observed in practice in a simulation example.

{\blockedit
\subsection{Proof for \cref{sec:lp}}

\begin{myproof}[Proof of \cref{thm: fast rate Lp}]
Let $A^{*}(s,a) = Q^{*} (s,a) - V^{*}(s)$, $\bar A(s,a) = \bar f(s,a)-\bar f(s,\pi^*(s))$ where we fix some choice of $\pi^*$, and $\mathcal A^*=\argmax_{a\in\A}Q^*(s,a)$. Also, for a function $g(s)$, and an $s$-distribution $\mu$, define $\|g\|_{p,\mu}^p=\expect_{s\sim\mu}\abs{g(s)}^p$.

By the performance difference lemma (\citet[Lemma 1.16]{agarwal2019reinforcement}), we have
$$
(1-\gamma) (V^*-V^{\bar\pi}) 
= \expect_{s\sim d^{\bar\pi}} \abs{A^*(s, \bar\pi(s))}
$$

Consider first $p<\infty$.
For any $t>0$, we have
\begin{align}
\expect_{s\sim d^{\bar\pi}} \abs{A^*(s, \bar\pi(s))}
=&~ \expect_{s\sim d^{\bar\pi}}[\abs{A^*(s, \bar\pi(s))}\indic{0<\abs{A^*(s, \bar\pi(s))}\leq t}]\label{eq: fast rate Lp 1}\\
&+ \expect_{s\sim d^{\bar\pi}}[\abs{A^*(s, \bar\pi(s))}\indic{\abs{A^*(s, \bar\pi(s))}> t}].\label{eq: fast rate Lp 2}
\end{align}

First, we bound \cref{eq: fast rate Lp 1} as follows:
\begin{align}
&~\expect_{s\sim d^{\bar\pi}}[\abs{A^*(s, \bar\pi(s))}\indic{0<\abs{A^*(s, \bar\pi(s))}\leq t}]\notag
\\\leq&~\expect_{s\sim d^{\bar\pi}}[\abs{A^*(s, \bar\pi(s))-\bar A(s, \bar\pi(s))}\indic{0<\abs{A^*(s, \bar\pi(s))}\leq t,\,\bar\pi(s)\notin\mathcal A^*}]\label{eq: fast rate Lp 1a}
\\\leq&~\magd{A^*(s, \bar\pi(s))-\bar A(s, \bar\pi(s))}_{p,d^{\bar\pi}}\pr_{s\sim d^{\bar\pi}}(0<\abs{A^*(s, \bar\pi(s))}\leq t,\,\bar\pi(s)\notin\mathcal A^*)^{1-1/p}\label{eq: fast rate Lp 1b}
\\\leq&~\magd{A^*(s, \bar\pi(s))-\bar A(s, \bar\pi(s))}_{p,d^{\bar\pi}}\pr_{s\sim d^{\bar\pi}}(0<\Delta(s)\leq t)^{1-1/p}\label{eq: fast rate Lp 1c}
\\\leq&~\magd{A^*(s, \bar\pi(s))-\bar A(s, \bar\pi(s))}_{p,d^{\bar\pi}}\prns{\frac t{\delta_0}}^{\alpha(p-1)/p}.\label{eq: fast rate Lp 1d}
\end{align}
In \cref{eq: fast rate Lp 1a} we use the fact that $A^*(s,\bar\pi(s))\leq0$ while $\bar A(s,\bar\pi(s))\geq0$. In \cref{eq: fast rate Lp 1b} we use H\"older's inequality. In \cref{eq: fast rate Lp 1c} we use $\Delta(s)\leq \abs{A^*(s,a)}$ provided $a\notin\mathcal A^*$. And, in \cref{eq: fast rate Lp 1d} we use \cref{ass: margin}.

Next, we bound \cref{eq: fast rate Lp 2} as follows:
\begin{align}
&~\expect_{s\sim d^{\bar\pi}}[\abs{A^*(s, \bar\pi(s))}\indic{\abs{A^*(s, \bar\pi(s))}> t}]\notag
\\\leq&~\expect_{s\sim d^{\bar\pi}}[\abs{A^*(s, \bar\pi(s))-\bar A(s, \bar\pi(s))}\indic{\abs{A^*(s, \bar\pi(s))-\bar A(s, \bar\pi(s))}> t}]\label{eq: fast rate Lp 2a}
\\\leq&~\magd{A^*(s, \bar\pi(s))-\bar A(s, \bar\pi(s))}_{p,d^{\bar\pi}}\pr_{s\sim d^{\bar\pi}}(\abs{A^*(s, \bar\pi(s))-\bar A(s, \bar\pi(s))}> t)^{1-p}\label{eq: fast rate Lp 2b}
\\\leq&~\magd{A^*(s, \bar\pi(s))-\bar A(s, \bar\pi(s))}_{p,d^{\bar\pi}}^pt^{1-p}.\label{eq: fast rate Lp 2c}
\end{align}
In \cref{eq: fast rate Lp 2a} we use the fact that $A^*(s,\bar\pi(s))\leq0$ while $\bar A(s,\bar\pi(s))\geq0$. In \cref{eq: fast rate Lp 2b} we use H\"older's inequality. And, in \cref{eq: fast rate Lp 2c} we use Markov's inequality.

Choosing $t=\magd{A^*(s, \bar\pi(s))-\bar A(s, \bar\pi(s))}_{p,d^{\bar\pi}}^{p/(p+\alpha)}\delta_0^{\alpha/(p+\alpha)}$, we obtain that
$$
V^*-V^{\bar\pi}\leq\frac2{1-\gamma}\delta_0^{(1-p)\alpha/(p+\alpha)}\magd{A^*(s, \bar\pi(s))-\bar A(s, \bar\pi(s))}_{p,d^{\bar\pi}}^{p(1+\alpha)/(p+\alpha)}.
$$
We conclude by noting that
$$
\magd{A^*(s, \bar\pi(s))-\bar A(s, \bar\pi(s))}_{p,d^{\bar\pi}}
\leq \magd{Q^*-\bar f}_{p,d^{\bar\pi}\times\bar\pi}+\magd{Q^*-\bar f}_{p,d^{\bar\pi}\times\pi^*}.
$$

Now consider $p=\infty$. Since $A^*(s,\bar\pi(s))\leq0$ while $\bar A(s,\bar\pi(s))\geq0$, we have that
\begin{align*}
\expect_{s\sim d^{\bar\pi}} &\abs{A^*(s, \bar\pi(s))}=
\expect_{s\sim d^{\bar\pi}} \abs{A^*(s, \bar\pi(s))}\indic{0<\abs{A^*(s, \bar\pi(s))}\leq\abs{A^*(s, \bar\pi(s))-\bar A(s, \bar\pi(s))},\,\bar\pi(s)\notin\mathcal A^*}
\\&\leq\expect_{s\sim d^{\bar\pi}} \abs{A^*(s, \bar\pi(s))-\bar A(s, \bar\pi(s))}\indic{0<\abs{A^*(s, \bar\pi(s))}\leq\abs{A^*(s, \bar\pi(s))-\bar A(s, \bar\pi(s))},\,\bar\pi(s)\notin\mathcal A^*}
\\&\leq \magd{A^*(s, \bar\pi(s))-\bar A(s, \bar\pi(s))}_\infty\pr_{s\sim d^{\bar\pi}}\prns{0<\abs{A^*(s, \bar\pi(s))}\leq\abs{A^*(s, \bar\pi(s))-\bar A(s, \bar\pi(s))},\,\bar\pi(s)\notin\mathcal A^*}
\\&\leq \magd{A^*(s, \bar\pi(s))-\bar A(s, \bar\pi(s))}_\infty\pr_{s\sim d^{\bar\pi}}\prns{0<\Delta(s)\leq\magd{A^*(s, \bar\pi(s))-\bar A(s, \bar\pi(s))}_\infty}
\\&\leq \magd{A^*(s, \bar\pi(s))-\bar A(s, \bar\pi(s))}_\infty^{1+\alpha}/\delta_0^\alpha.
\end{align*}
We conclude by noting that
$$
\magd{A^*(s, \bar\pi(s))-\bar A(s, \bar\pi(s))}_{\infty}
\leq 2\magd{Q^*-\bar f}_{\infty}.
$$
\end{myproof}

\begin{myproof}[Proof of \cref{lem:FQI}]
We begin by reusing the arguments in the proof of lemma 4.4 of \citet{agarwal2019reinforcement}. First, suppose 
\begin{align*}
    \|\hat f_{k+1}- \Tcal \hat f_k \|_{2,\mu_b} \leq \epsilon 
\end{align*}
for any $k=1,\cdots,K-1$.

Let $\beta(s,a;\pi_1,\cdots,\pi_t)$ be the state-action distribution after executing $\pi_1,\cdots,\pi_t$ after starting with an initial state-action distribution $\beta$. Then, from the proof of lemma 4.4 in \citet{agarwal2019reinforcement}, for any $\beta$, we have  
\begin{align*}
    \|Q^*-\hat f_k\|_{2,\beta} \leq  \gamma \|Q^*-\hat f_{k-1}\|_{2,\beta(\cdot;\pi)}  +   \|Q^*-\hat f_k\|_{2,\beta} 
\end{align*}
for a certain choice of $\pi$. 
Therefore, given any policy $\pi\in\Pi$ and arbitrary policy $\pi'$, defining the state-action distribution $\beta=d^\pi\times\pi'$, we have that for some certain choices of policies $\pi_1,\cdots,\pi_K$, we can obtain the following recursion:
\begin{align*}
     &\|Q^*-\hat f_k\|_{2,\beta}\\
     &\leq \gamma \|Q^*-\hat f_{k-1}\|_{2,\beta(\cdot;\pi_1)}  +   \|Q^*-\hat f_k\|_{2,\beta} \\
     &\leq   \gamma \|Q^*-\hat f_{k-1}\|_{2,\beta(\cdot;\pi_1)}  +   \sqrt{C_{\mathrm{all}}(\Pi)} \|Q^*-\hat f_k\|_{2,\mu_b} \\ 
      &\leq  \gamma\{ \gamma \|Q^*-\hat f_{k-2}\|_{2,\beta(\cdot;\pi_1,\pi_2)} +  \|Q^*-\hat f_k\|_{2,\beta(\cdot;\pi_1) } \}  +   \sqrt{ C_{\mathrm{all}}(\Pi)} \|Q^*-\hat f_k\|_{2,\mu_b} \\
    & \leq  \gamma\{ \gamma \|Q^*-\hat f_{k-2}\|_{2,\beta(\cdot;\pi_1,\pi_2)} + \sqrt{C_{\mathrm{all}}(\Pi) } \|Q^*-\hat f_k\|_{2,\mu_b} \}  +   \sqrt{ C_{\mathrm{all}}(\Pi)} \|Q^*-\hat f_k\|_{2,\mu_b}\\
    & \ldots \\
    & \leq \sqrt{ C_{\mathrm{all}}(\Pi)}\frac{\epsilon}{1-\gamma} + 2\gamma^K\|\Fcal\|_{\infty}. 
\end{align*}

Lastly, from \citet[Theorem 14.20]{wainwright2019high} and noting that $|r+\gamma\max_{a'}\hat f_k(s',a')|\leq 2M$, we obtain that with probability $1-\delta$, for any $k \in [1,\cdots,K-1]$,
\begin{align*}
    \|\hat f_{k+1} - \Tcal \hat f_k\|_{2,\mu_b}\leq c_1M(\rho_n + \sqrt{\log(K/\delta)/n}). 
\end{align*}
Then, by replacing $\epsilon$ with the above error bound, the statement is immediately concluded.
\end{myproof}

\begin{myproof}[Proof of \cref{thm: fast rate Lp2_single}]
By the performance difference lemma (\citet[Lemma 1.16]{agarwal2019reinforcement}), we have for any $t>0$ that
\begin{align}
    &(1-\gamma)^2 R_{\max}^{-1}\{V^{*}-V^{\bar \pi }\} \notag\\
    &\leq \mathbb{E}_{s \sim d_{\pi^{\star}}}[\indic{\pi^{\star}(s)\neq \bar \pi (s)} ] \notag \\
    &\leq\mathbb{E}_{s \sim d_{\pi^{\star}}}\bracks{\sum_{a'\in \mathcal{A}} \indic{\bar f(s,a')-\bar f(s,\pi^{\star}(s))\geq 0,\,Q^{\star}(s,a')-Q^{\star}(s,\pi^{\star}(s))\leq 0 }  } \notag\\
    &\leq \mathbb{E}_{s \sim d_{\pi^{\star}}}\bracks{\sum_{a'\in \mathcal{A}} \indic{Q^{\star}(s,a')-Q^{\star}(s,\pi^{\star}(s))\geq -t }  } \label{eq: fast rate Lp2_single 1}\\
     &\phantom{\leq} + \mathbb{E}_{s \sim d_{\pi^{\star}}}\bracks{\sum_{a'\in \mathcal{A}} \indic{\bar f(s,a')-\bar f(s,\pi^{\star}(s))-Q^{\star}(s,a')+Q^{\star}(s,\pi^{\star}(s))\geq t }  }\label{eq: fast rate Lp2_single 2}
\end{align}
To bound \cref{eq: fast rate Lp2_single 1} we use the margin assumption:
\begin{align*}
    \mathbb{E}_{s \sim d_{\pi^{\star}}}\bracks{\sum_{a'\in \mathcal{A}} \indic{Q^{\star}(s,a')-Q^{\star}(s,\pi^{\star}(s))\geq -t }  }\leq |\Acal|(t/\delta_0)^{\alpha}. 
\end{align*}
Next, we bound \cref{eq: fast rate Lp2_single 2}:
\begin{align*}
    &t^p \mathbb{E}_{s \sim d_{\pi^{\star}}}\bracks{\sum_{a'\in \mathcal{A}}  \indic{\bar f(s,a')-\bar f(s,\pi^{\star}(s))-Q^{\star}(s,a')+Q^{\star}(s,\pi^{\star}(s))\geq t }  } \\
    &\leq  \mathbb{E}_{s \sim d_{\pi^{\star}}}\bracks{\sum_{a'\in \mathcal{A}} |\bar f(s,a')-\bar f(s,\pi^{\star}(s))-Q^{\star}(s,a')+Q^{\star}(s,\pi^{\star}(s))|^p } \tag{Markov inequality} \\
  &\leq 2^p \mathbb{E}_{s \sim d_{\pi^{\star}}}\bracks{\sum_{a'\in \mathcal{A}} |\bar f(s,a')-Q^{\star}(s,a')|^p + |\bar f(s,\pi^{\star}(s))-Q^{\star}(s,\pi^{\star}(s))|^p }\\
    &\leq 2^p|\Acal| \mathbb{E}_{s \sim d_{\pi^{\star}},a'\sim\pi^\mathrm{Uni}}\bracks{ |\bar f(s,a')-Q^{\star}(s,a')|^p + |\bar f(s,\pi^{\star}(s))-Q^{\star}(s,\pi^{\star}(s))|^p }.  
\end{align*}
Therefore, choosing $t$ to equate the bound on the two terms,
\begin{align*}
&(1-\gamma)^2 M^{-1}\{V^{*}-V^{\bar \pi }\}\\
&\leq |\Acal|(t/\delta_0)^{\alpha} +t^{-p} 2^p|\Acal| \mathbb{E}_{s \sim d_{\pi^{\star}},a'\sim \pi^\mathrm{Uni}}\bracks{ |\hat q(s,a')-Q^{\star}(s,a')|^p + |\hat q(s,\pi^{\star}(s))-Q^{\star}(s,\pi^{\star}(s))|^p }  \\
&\leq 2|\Acal| 2^{\frac{p \alpha}{\alpha + p}} \delta_0^{ \frac{p \alpha}{\alpha + p}} ( \mathbb{E}_{s \sim d_{\pi^{\star}},a'\sim\pi^\mathrm{Uni}}\bracks{ |\hat q(s,a')-Q^{\star}(s,a')|^p + |\hat q(s,\pi^{\star}(s))-Q^{\star}(s,\pi^{\star}(s))|^p } )^{\frac{\alpha}{\alpha + p} }  \\ 
&\leq 2|\Acal| 2^{\frac{p \alpha}{\alpha + p}} \delta_0^{ \frac{p \alpha}{\alpha + p}}  C_{\mathrm{sin}}^{\frac{\alpha}{\alpha +p}} \|Q^{\star}-\bar f \|^{\frac{p\alpha}{(\alpha +p)}}_{p}. 
\end{align*}
\end{myproof}
}

\section*{Acknowledgments.}
Th authors thank Yoav Kallus for providing the geometric argument to improve the dependence of the $\delta_0$ constant on $\abs{\A}$ in the proof of \cref{lemma:linearmargin}, which is based on a proof by Noam Elkies of the monotonicity of perimeter for convex sets under containment (\url{https://mathoverflow.net/questions/71502/circumference-of-convex-shapes/71505#71505}).
The authors acknowledge the helpful and constructive comments by the anonymous reviewers from COLT 2021, where an extended abstract of this work was presented.
This material is based upon work supported by the National Science Foundation under Grant No. 1846210.

\bibliographystyle{plainnat}

\bibliography{literature}

\begin{thebibliography}{45}
\providecommand{\natexlab}[1]{#1}
\providecommand{\url}[1]{\texttt{#1}}
\expandafter\ifx\csname urlstyle\endcsname\relax
  \providecommand{\doi}[1]{doi: #1}\else
  \providecommand{\doi}{doi: \begingroup \urlstyle{rm}\Url}\fi

\bibitem[Agarwal et~al.(2020{\natexlab{a}})Agarwal, Jiang, Kakade, and
  Sun]{agarwal2019reinforcement}
Alekh Agarwal, Nan Jiang, Sham~M Kakade, and Wen Sun.
\newblock Reinforcement learning: Theory and algorithms.
\newblock Technical report, 2020{\natexlab{a}}.

\bibitem[Agarwal et~al.(2020{\natexlab{b}})Agarwal, Kakade, and
  Yang]{pmlr-v125-agarwal20b}
Alekh Agarwal, Sham Kakade, and Lin~F. Yang.
\newblock Model-based reinforcement learning with a generative model is minimax
  optimal.
\newblock In Jacob Abernethy and Shivani Agarwal, editors, \emph{Proceedings of
  Thirty Third Conference on Learning Theory}, volume 125 of \emph{Proceedings
  of Machine Learning Research}, pages 67--83, 09--12 Jul 2020{\natexlab{b}}.

\bibitem[Antos et~al.(2008)Antos, Szepesv{\'a}ri, and Munos]{antos2008learning}
Andr{\'a}s Antos, Csaba Szepesv{\'a}ri, and R{\'e}mi Munos.
\newblock Learning near-optimal policies with bellman-residual minimization
  based fitted policy iteration and a single sample path.
\newblock \emph{Machine Learning}, 71\penalty0 (1):\penalty0 89--129, 2008.

\bibitem[Audibert and Tsybakov(2007)]{audibert2007fast}
Jean-Yves Audibert and Alexandre~B Tsybakov.
\newblock Fast learning rates for plug-in classifiers.
\newblock \emph{The Annals of statistics}, 35\penalty0 (2):\penalty0 608--633,
  2007.

\bibitem[Bastani and Bayati(2020)]{bastani2020online}
Hamsa Bastani and Mohsen Bayati.
\newblock Online decision making with high-dimensional covariates.
\newblock \emph{Operations Research}, 68\penalty0 (1):\penalty0 276--294, 2020.

\bibitem[Chen and Jiang(2019)]{ChenJinglin2019ICiB}
Jinglin Chen and Nan Jiang.
\newblock Information-theoretic considerations in batch reinforcement learning.
\newblock In \emph{Proceedings of the 36th International Conference on Machine
  Learning}, volume~97, pages 1042--1051, 2019.

\bibitem[Dai et~al.(2018)Dai, Shaw, Li, Xiao, He, Liu, Chen, and
  Song]{pmlr-v80-dai18c}
Bo~Dai, Albert Shaw, Lihong Li, Lin Xiao, Niao He, Zhen Liu, Jianshu Chen, and
  Le~Song.
\newblock {SBEED}: Convergent reinforcement learning with nonlinear function
  approximation.
\newblock In \emph{Proceedings of the 35th International Conference on Machine
  Learning}, pages 1125--1134, 2018.

\bibitem[Du et~al.(2019)Du, Luo, Wang, and Zhang]{du2019provably}
Simon~S Du, Yuping Luo, Ruosong Wang, and Hanrui Zhang.
\newblock Provably efficient q-learning with function approximation via
  distribution shift error checking oracle.
\newblock \emph{Advances in Neural Information Processing Systems}, 32, 2019.

\bibitem[Du et~al.(2020)Du, Kakade, Wang, and Yang]{du2020good}
Simon~S Du, Sham~M Kakade, Ruosong Wang, and Lin~F Yang.
\newblock Is a good representation sufficient for sample efficient
  reinforcement learning?
\newblock In \emph{International Conference on Learning Representations}, 2020.

\bibitem[Duan et~al.(2020)Duan, Jia, and Wang]{DuanYaqi2020MOEw}
Yaqi Duan, Zeyu Jia, and Mengdi Wang.
\newblock Minimax-optimal off-policy evaluation with linear function
  approximation.
\newblock In \emph{Proceedings of the 37th International Conference on Machine
  Learning}, volume 119 of \emph{Proceedings of Machine Learning Research},
  pages 2701--2709, 2020.

\bibitem[Ernst et~al.(2005)Ernst, Geurts, and Wehenkel]{ernst2005tree}
Damien Ernst, Pierre Geurts, and Louis Wehenkel.
\newblock Tree-based batch mode reinforcement learning.
\newblock \emph{Journal of Machine Learning Research}, 6\penalty0
  (Apr):\penalty0 503--556, 2005.

\bibitem[Fan et~al.(2020)Fan, Wang, Xie, and Yang]{FanJianqing2019ATAo}
Jianqing Fan, Zhaoran Wang, Yuchen Xie, and Zhuoran Yang.
\newblock A theoretical analysis of deep q-learning.
\newblock In \emph{Proceedings of the 2nd Conference on Learning for Dynamics
  and Control}, volume 120 of \emph{Proceedings of Machine Learning Research},
  pages 486--489, 2020.

\bibitem[Gheshlaghi~Azar et~al.(2013)Gheshlaghi~Azar, Munos, and
  Kappen]{edssjs.BD18D58A20130101}
Mohammad Gheshlaghi~Azar, Rémi Munos, and Hilbert~J. Kappen.
\newblock Minimax pac bounds on the sample complexity of reinforcement learning
  with a generative model.
\newblock \emph{Machine Learning}, 91\penalty0 (3):\penalty0 325, 2013.
\newblock ISSN 0885-6125.

\bibitem[Goldenshluger and Zeevi(2013)]{goldenshluger2013linear}
Alexander Goldenshluger and Assaf Zeevi.
\newblock A linear response bandit problem.
\newblock \emph{Stochastic Systems}, 3\penalty0 (1):\penalty0 230--261, 2013.

\bibitem[Hu et~al.(2022{\natexlab{a}})Hu, Kallus, and Mao]{hu2022fast}
Yichun Hu, Nathan Kallus, and Xiaojie Mao.
\newblock Fast rates for contextual linear optimization.
\newblock \emph{Management Science}, 68\penalty0 (6):\penalty0 4236--4245,
  2022{\natexlab{a}}.

\bibitem[Hu et~al.(2022{\natexlab{b}})Hu, Kallus, and Mao]{hu2022smooth}
Yichun Hu, Nathan Kallus, and Xiaojie Mao.
\newblock Smooth contextual bandits: Bridging the parametric and
  nondifferentiable regret regimes.
\newblock \emph{Operations Research}, 70\penalty0 (6):\penalty0 3261--3281,
  2022{\natexlab{b}}.

\bibitem[Jin et~al.(2020)Jin, Yang, Wang, and Jordan]{jin2020provably}
Chi Jin, Zhuoran Yang, Zhaoran Wang, and Michael~I Jordan.
\newblock Provably efficient reinforcement learning with linear function
  approximation.
\newblock In \emph{Conference on Learning Theory}, pages 2137--2143. PMLR,
  2020.

\bibitem[Kallus and Uehara(2020{\natexlab{a}})]{EEOPG2020}
Nathan Kallus and Masatoshi Uehara.
\newblock Statistically efficient off-policy policy gradients.
\newblock In \emph{Proceedings of the 37th International Conference on Machine
  Learning}, volume 119, pages 5089--5100, 2020{\natexlab{a}}.

\bibitem[Kallus and Uehara(2020{\natexlab{b}})]{kallus2020double}
Nathan Kallus and Masatoshi Uehara.
\newblock Double reinforcement learning for efficient off-policy evaluation in
  markov decision processes.
\newblock \emph{Journal of Machine Learning Research}, 21\penalty0
  (167):\penalty0 1--63, 2020{\natexlab{b}}.

\bibitem[Kallus and Uehara(2022)]{kallus2022efficiently}
Nathan Kallus and Masatoshi Uehara.
\newblock Efficiently breaking the curse of horizon in off-policy evaluation
  with double reinforcement learning.
\newblock \emph{Operations Research}, 70\penalty0 (6):\penalty0 3282--3302,
  2022.

\bibitem[Lagoudakis and Parr(2004)]{LagoudakisMichail2004LPI}
Michail Lagoudakis and Ronald Parr.
\newblock Least-squares policy iteration.
\newblock \emph{Journal of Machine Learning Research}, 4\penalty0 (6):\penalty0
  1107--1149, 2004.

\bibitem[Lazaric et~al.(2010)Lazaric, Ghavamzadeh, and
  Munos]{LazaricAlessandro2010FAoL}
Alessandro Lazaric, Mohammad Ghavamzadeh, and Remi Munos.
\newblock Finite-sample analysis of lstd.
\newblock pages 615--622, 2010.

\bibitem[Liu et~al.(2018)Liu, Li, Tang, and Zhou]{liu2018breaking}
Qiang Liu, Lihong Li, Ziyang Tang, and Dengyong Zhou.
\newblock Breaking the curse of horizon: infinite-horizon off-policy
  estimation.
\newblock In \emph{Proceedings of the 32nd International Conference on Neural
  Information Processing Systems}, pages 5361--5371, 2018.

\bibitem[Liu et~al.(2020)Liu, Swaminathan, Agarwal, and
  Brunskill]{liu2020provably}
Yao Liu, Adith Swaminathan, Alekh Agarwal, and Emma Brunskill.
\newblock Provably good batch off-policy reinforcement learning without great
  exploration.
\newblock \emph{Advances in neural information processing systems},
  33:\penalty0 1264--1274, 2020.

\bibitem[Luedtke and Chambaz(2020)]{luedtke2020performance}
Alex Luedtke and Antoine Chambaz.
\newblock Performance guarantees for policy learning.
\newblock In \emph{Annales de l'IHP Probabilites et statistiques}, volume~56,
  page 2162. NIH Public Access, 2020.

\bibitem[Mammen and Tsybakov(1999)]{mammen1999smooth}
Enno Mammen and Alexandre~B Tsybakov.
\newblock Smooth discrimination analysis.
\newblock \emph{The Annals of Statistics}, 27\penalty0 (6):\penalty0
  1808--1829, 1999.

\bibitem[Munos(2003)]{munos2003error}
R{\'e}mi Munos.
\newblock Error bounds for approximate policy iteration.
\newblock In \emph{ICML}, volume~3, pages 560--567, 2003.

\bibitem[Munos(2005)]{munos2005error}
R{\'e}mi Munos.
\newblock Error bounds for approximate value iteration.
\newblock In \emph{Proceedings of the National Conference on Artificial
  Intelligence}, volume~20, page 1006. Menlo Park, CA; Cambridge, MA; London;
  AAAI Press; MIT Press; 1999, 2005.

\bibitem[Munos and Szepesv{\'a}ri(2008)]{munos2008finite}
R{\'e}mi Munos and Csaba Szepesv{\'a}ri.
\newblock Finite-time bounds for fitted value iteration.
\newblock \emph{Journal of Machine Learning Research}, 9\penalty0
  (May):\penalty0 815--857, 2008.

\bibitem[Nachum et~al.(2019)Nachum, Dai, Kostrikov, Chow, Li, and
  Schuurmans]{NachumOfir2019APGf}
Ofir Nachum, Bo~Dai, Ilya Kostrikov, Yinlam Chow, Lihong Li, and Dale
  Schuurmans.
\newblock Algaedice: Policy gradient from arbitrary experience.
\newblock \emph{arXiv preprint arXiv:1912.02074}, 2019.

\bibitem[Perchet and Rigollet(2013)]{perchet2013multi}
Vianney Perchet and Philippe Rigollet.
\newblock The multi-armed bandit problem with covariates.
\newblock \emph{Annals of statistics}, 41\penalty0 (2):\penalty0 693--721,
  2013.

\bibitem[Pollard(1990)]{pollard1990empirical}
David Pollard.
\newblock Empirical processes: theory and applications.
\newblock In \emph{NSF-CBMS regional conference series in probability and
  statistics}, pages i--86. JSTOR, 1990.

\bibitem[Precup et~al.(2000)Precup, Sutton, and Singh]{Precup2000}
D.~Precup, R.~Sutton, and S~Singh.
\newblock Eligibility traces for off-policy policy evaluation.
\newblock \emph{In Proceedings of the 17th International Conference on Machine
  Learning}, pages 759--766, 2000.

\bibitem[Scherrer(2014)]{pmlr-v32-scherrer14}
Bruno Scherrer.
\newblock Approximate policy iteration schemes: A comparison.
\newblock In Eric~P. Xing and Tony Jebara, editors, \emph{Proceedings of the
  31st International Conference on Machine Learning}, volume~32 of
  \emph{Proceedings of Machine Learning Research}, pages 1314--1322, 2014.

\bibitem[Simchowitz and Jamieson(2019)]{simchowitz2019non}
Max Simchowitz and Kevin~G Jamieson.
\newblock Non-asymptotic gap-dependent regret bounds for tabular mdps.
\newblock \emph{Advances in Neural Information Processing Systems},
  32:\penalty0 1153--1162, 2019.

\bibitem[Singh and Yee(1994)]{SinghSatinderP1994Aubo}
Satinder~P Singh and Richard~C Yee.
\newblock An upper bound on the loss from approximate optimal-value functions.
\newblock 16\penalty0 (3):\penalty0 227--233, 1994.

\bibitem[Thomas and Brunskill(2016)]{thomas2016}
P.~Thomas and E.~Brunskill.
\newblock Data-efficient off-policy policy evaluation for reinforcement
  learning.
\newblock \emph{In Proceedings of the 33rd International Conference on Machine
  Learning}, pages 2139--2148, 2016.

\bibitem[Tropp(2015)]{tropp2015introduction}
Joel~A Tropp.
\newblock An introduction to matrix concentration inequalities.
\newblock \emph{Foundations and Trends{\textregistered} in Machine Learning},
  8\penalty0 (1-2):\penalty0 1--230, 2015.

\bibitem[Tsybakov(2004)]{tsybakov2004optimal}
Alexander~B Tsybakov.
\newblock Optimal aggregation of classifiers in statistical learning.
\newblock \emph{The Annals of Statistics}, 32\penalty0 (1):\penalty0 135--166,
  2004.

\bibitem[Uehara et~al.(2023)Uehara, Kallus, Lee, and Sun]{uehara2023refined}
Masatoshi Uehara, Nathan Kallus, Jason~D Lee, and Wen Sun.
\newblock Refined value-based offline rl under realizability and partial
  coverage.
\newblock \emph{arXiv preprint arXiv:2302.02392}, 2023.

\bibitem[Wainwright(2019)]{wainwright2019high}
Martin~J Wainwright.
\newblock \emph{High-dimensional statistics: A non-asymptotic viewpoint},
  volume~48.
\newblock Cambridge University Press, 2019.

\bibitem[Wang et~al.(2021)Wang, Foster, and Kakade]{wang2021statistical}
Ruosong Wang, Dean Foster, and Sham~M Kakade.
\newblock What are the statistical limits of offline rl with linear function
  approximation?
\newblock In \emph{International Conference on Learning Representations}, 2021.

\bibitem[Xie and Jiang(2020)]{XieTengyang2020QASf}
Tengyang Xie and Nan Jiang.
\newblock Q* approximation schemes for batch reinforcement learning: A
  theoretical comparison.
\newblock \emph{UAI2020}, 2020.

\bibitem[Yang et~al.(2021)Yang, Yang, and Du]{yang2021q}
Kunhe Yang, Lin Yang, and Simon Du.
\newblock Q-learning with logarithmic regret.
\newblock In \emph{International Conference on Artificial Intelligence and
  Statistics}, pages 1576--1584. PMLR, 2021.

\bibitem[Yin et~al.(2021)Yin, Bai, and Wang]{yin2021near}
Ming Yin, Yu~Bai, and Yu{-}Xiang Wang.
\newblock Near-optimal provable uniform convergence in offline policy
  evaluation for reinforcement learning.
\newblock In Arindam Banerjee and Kenji Fukumizu, editors, \emph{The 24th
  International Conference on Artificial Intelligence and Statistics, {AISTATS}
  2021, April 13-15, 2021, Virtual Event}, volume 130 of \emph{Proceedings of
  Machine Learning Research}, pages 1567--1575. {PMLR}, 2021.

\end{thebibliography}

\end{document}